\documentclass[a4paper]{article}
\usepackage{times}
\usepackage{microtype}
\usepackage{graphicx}
\usepackage{subcaption}
\usepackage{fullpage}
\usepackage{booktabs} 
\usepackage[most]{tcolorbox}

\usepackage{hyperref}

\usepackage{amsmath}
\usepackage{amssymb}
\usepackage{mathtools}
\usepackage{amsthm}

\usepackage{thmtools}
\usepackage{thm-restate}
\usepackage[capitalize,noabbrev]{cleveref}

\theoremstyle{plain}
\newtheorem{theorem}{Theorem}[section]

\newtheorem{lemma}[theorem]{Lemma}

\theoremstyle{definition}

\newtheorem{assumption}[theorem]{Assumption}
\theoremstyle{remark}

\usepackage{enumitem}

\newcommand*\samethanks[1][\value{footnote}]{\footnotemark[#1]}

\title{Interaction-Grounded Learning for Contextual Markov Decision Processes with Personalized Feedback}
\author{
    Mengxiao Zhang\thanks{Equal contribution.}\\
    University of Iowa\\
    \texttt{mengxiao-zhang@uiowa.edu}
    \and
    Yuheng Zhang\samethanks{}\\
    University of Illinois Urbana-Champaign \\
    \texttt{yuhengz2@illinois.edu}
    \and 
\hspace*{-12pt} Haipeng Luo\\
\hspace*{-12pt} University of Southern California\\
\hspace*{-12pt}\texttt{haipengl@usc.edu}
\and 
\hspace*{20pt} Paul Mineiro\\
\hspace*{20pt} Microsoft Research\\
\hspace*{20pt} \texttt{pmineiro@microsoft.com}
}

\newif\ifspacehack
\usepackage{natbib}
\hypersetup{
    colorlinks = blue,
    breaklinks,
    linkcolor = blue,
    citecolor = blue,
    urlcolor  = blue,
}
\usepackage{url} 
\usepackage{graphicx}
\usepackage{mathtools}
\usepackage{footnote}
\usepackage{float}
\usepackage{amsthm}
\usepackage{xspace}
\usepackage{multirow}
\usepackage{xcolor}
\usepackage{wrapfig}
\usepackage{framed}
\usepackage{bbm}
\usepackage{footnote}
\usepackage{nicefrac}
\usepackage{makecell}
\usepackage[vlined, algo2e]{algorithm2e} 
\usepackage{algorithm}
\usepackage{amssymb}
\usepackage{bm}
\makesavenoteenv{tabular}
\makesavenoteenv{table}

\renewcommand{\tilde}{\widetilde}

\def \Pr {\mathsf{Pr}}

\newcommand{\eps}{\epsilon}

\newcommand{\calA}{{\mathcal{A}}}

\newcommand{\calB}{{\mathcal{B}}}
\newcommand{\calX}{{\mathcal{X}}}
\newcommand{\calS}{{\mathcal{S}}}
\newcommand{\calF}{{\mathcal{F}}}

\newcommand{\calH}{{\mathcal{H}}}
\newcommand{\calD}{{\mathcal{D}}}
\newcommand{\calE}{{\mathcal{E}}}

\newcommand{\calT}{{\mathcal{T}}}
\newcommand{\calP}{{\mathcal{P}}}

\newcommand{\calY}{{\mathcal{Y}}}

\newcommand{\Reg}{\mathrm{\mathbf{Reg}}}

\newcommand{\euler}{\textsc{Euler}}

\newcommand{\E}{{\mathbb{E}}}

\newcommand{\order}{\mathcal{O}}

\newcommand{\Acal}{\mathcal{A}}

\newcommand{\Scal}{\mathcal{S}}

\newcommand{\unif}{\text{\rm unif}}

\DeclareMathOperator*{\argmin}{argmin}
\DeclareMathOperator*{\argmax}{argmax}

\newcommand{\wh}{\widehat}
\newcommand{\wt}{\widetilde}

\newcommand{\AlgSq}{\ensuremath{\mathsf{AlgSq}}\xspace}

\newcommand{\otil}{\ensuremath{\tilde{\mathcal{O}}}}

\usepackage{lipsum,booktabs}
\usepackage{amsmath,mathrsfs,amssymb,amsfonts,bm,enumitem}
\usepackage{rotating}
\usepackage{pdflscape}
\usepackage{hyperref,url}
\hypersetup{
    colorlinks,
    breaklinks,
    linkcolor = blue,
    citecolor = blue,
    urlcolor  = blue,
}
\allowdisplaybreaks
\usepackage{appendix}
\usepackage{multirow,makecell}

\usepackage{algorithmic,algorithm}
\renewcommand{\algorithmicrequire}{ \textbf{Input:}}
\renewcommand{\algorithmicensure}{ \textbf{Output:}}

\renewcommand{\tilde}{\widetilde}

\def \E {\mathbb{E}}

\def \Pr {\mathsf{Prob}}

\newcommand{\RegSq}{\ensuremath{\mathrm{\mathbf{Reg}}_{\mathsf{Sq}}}\xspace}

\usepackage{mathtools}

\usepackage{graphicx,color}

\definecolor{wine_red}{RGB}{228,48,64}
\definecolor{DSgray}{cmyk}{0,1,0,0}

\usepackage{prettyref}
\newcommand{\pref}[1]{\prettyref{#1}}

\newcommand{\savehyperref}[2]{\texorpdfstring{\hyperref[#1]{#2}}{#2}}
\newrefformat{eq}{\savehyperref{#1}{Eq. \textup{(\ref*{#1})}}}
\newrefformat{eqn}{\savehyperref{#1}{Eq.~(\ref*{#1})}}
\newrefformat{lem}{\savehyperref{#1}{Lemma~\ref*{#1}}}
\newrefformat{def}{\savehyperref{#1}{Definition~\ref*{#1}}}
\newrefformat{line}{\savehyperref{#1}{Line~\ref*{#1}}}
\newrefformat{thm}{\savehyperref{#1}{Theorem~\ref*{#1}}}
\newrefformat{corr}{\savehyperref{#1}{Corollary~\ref*{#1}}}
\newrefformat{cor}{\savehyperref{#1}{Corollary~\ref*{#1}}}
\newrefformat{sec}{\savehyperref{#1}{Section~\ref*{#1}}}
\newrefformat{app}{\savehyperref{#1}{Appendix~\ref*{#1}}}
\newrefformat{assum}{\savehyperref{#1}{Assumption~\ref*{#1}}}
\newrefformat{asm}{\savehyperref{#1}{Assumption~\ref*{#1}}}
\newrefformat{ex}{\savehyperref{#1}{Example~\ref*{#1}}}
\newrefformat{fig}{\savehyperref{#1}{Figure~\ref*{#1}}}
\newrefformat{alg}{\savehyperref{#1}{Algorithm~\ref*{#1}}}
\newrefformat{rem}{\savehyperref{#1}{Remark~\ref*{#1}}}
\newrefformat{conj}{\savehyperref{#1}{Conjecture~\ref*{#1}}}
\newrefformat{prop}{\savehyperref{#1}{Proposition~\ref*{#1}}}
\newrefformat{proto}{\savehyperref{#1}{Protocol~\ref*{#1}}}
\newrefformat{prob}{\savehyperref{#1}{Problem~\ref*{#1}}}
\newrefformat{claim}{\savehyperref{#1}{Claim~\ref*{#1}}}
\newrefformat{que}{\savehyperref{#1}{Question~\ref*{#1}}}
\newrefformat{op}{\savehyperref{#1}{Open Problem~\ref*{#1}}}
\newrefformat{fn}{\savehyperref{#1}{Footnote~\ref*{#1}}}

\def \epsilon {\varepsilon}

\newtcolorbox{promptbox}[1][]{
  colback=blue!5!white, colframe=blue!75!black,
  fonttitle=\bfseries, title=Prompt,
  left=2mm, right=2mm, top=2mm, bottom=2mm,
  boxrule=0.5mm,  
  coltitle=black, 
  colbacktitle=blue!15!white, 
  breakable,      
  #1
}

\begin{document}
\maketitle

\begin{abstract}

In this paper, we study Interaction-Grounded Learning (IGL)~\citep{xie2021interaction}, a paradigm designed for realistic scenarios where the learner receives indirect feedback generated by an unknown mechanism, rather than explicit numerical rewards. While prior work on IGL provides efficient algorithms with provable guarantees, those results are confined to single-step settings, restricting their applicability to modern sequential decision-making systems such as multi-turn Large Language Model (LLM) deployments. 
To bridge this gap, we propose a computationally efficient algorithm that achieves a sublinear regret guarantee for contextual episodic Markov Decision Processes (MDPs) with personalized feedback. Technically, we extend the reward-estimator construction of \citet{zhang2024efficient} from the single-step to the multi-step setting, addressing the unique challenges of decoding latent rewards under MDPs. Building on this estimator, we design an Inverse-Gap-Weighting (IGW) algorithm for policy optimization. Finally, we demonstrate the effectiveness of our method in learning personalized objectives from multi-turn interactions through experiments on both a synthetic episodic MDP and a real-world user booking dataset.
\end{abstract}
\section{Introduction}\label{sec: intro}

While classic formulations of interactive learning assume access to direct numerical rewards,
many modern interactive systems rarely receive such clean supervision.
Instead, they must rely on indirect, user-facing feedback, such as thumbs-up signals in social media~\citep{neophytou2022revisiting}, short textual comments on review platforms~\citep{hasan2025based}, electroencephalogram (EEG) signals in brain-computer interfaces~\citep{akinola2020accelerated}, or follow-up questions in conversational search that implicitly reveal satisfaction or dissatisfaction~\citep{kim2024followup,lin2024interpretable}. 
This phenomenon is especially salient for Large Language Models (LLMs) deployed in conversational settings: the system produces a sequence of intermediate outputs over multiple turns, yet users typically evaluate the interaction holistically based on the final response quality and its alignment with their personal preferences.

Interaction-Grounded Learning (IGL), first proposed by~\citet{xie2021interaction}, offers a principled framework to model learning from such indirect feedback.
Rather than receiving a direct reward as the feedback, IGL posits an underlying latent reward and views the observed feedback as a noisy proxy generated by an unknown mechanism.
To make learning tractable, earlier works imposed many structural assumptions on the relationship between the feedback and the latent reward, such as requiring feedback to be conditionally independent of the context given the latent reward~\citep{xie2021interaction, xie2022interaction}.
While these assumptions enable provable guarantees, they are often too restrictive for practical settings.
For instance, different users may express feedback in systematically different ways for the same outcome, or a single user may react differently depending on the conversational context.
This highlights a critical need for a context-dependent grounding mechanism, rather than a universal mapping.
Motivated by this, recent work has generalized the IGL framework to incorporate context-dependent feedback structures~\citep{maghakian2023personalized,zhang2024provably}.

Despite this progress, to the best of our knowledge, all existing IGL literature is restricted to the single-step contextual bandit scenario.
This abstraction misses the defining feature of many contemporary applications: interactions are inherently sequential and multi-step.
In an LLM-based system, for example, an agent may ask clarifying questions, invoke tools, or decide how much evidence to gather—decisions that fundamentally shape the trajectory of the conversation—before receiving any feedback.
This combination of multi-step sequential decision-making and indirect feedback creates a fundamental gap between current IGL theory and the needs of modern systems.
This motivates our central research question:
\begin{center}
\it Can we design provably efficient algorithms for interaction-grounded learning in multi-step contextual Markov Decision Processes (MDPs) with personalized feedback?
\end{center}

\subsection{Our Contribution}\label{sec: contribution}
In this work, we answer this question affirmatively.
We propose a computationally efficient algorithm for IGL in contextual episodic MDPs with personalized feedback and establish a sublinear regret guarantee.
At a high level, our approach decomposes this problem into \emph{reward estimator learning} and \emph{policy learning}.
In \pref{sec: decoder}, we develop a novel procedure for constructing a reward estimator from indirect feedback, comprising three main components:
(i) \emph{Reachable state identification}, which isolates terminal states that are reliably reachable by specific policies to ensure sufficient feedback data collection;
(ii) \emph{Inverse kinematic learning}, which estimates the posterior distribution of actions under a uniform policy conditioned on the reachable states identified in the previous step; and
(iii) \emph{A Lipschitz reward estimator}, which provides a smooth approximation of the latent reward function by leveraging the learned inverse kinematic at each reachable state.
This estimator non-trivially generalizes the contextual bandit construction of \citet{zhang2024provably} to the multi-step setting.
Building upon this estimator, in \pref{sec: regret}, we introduce a policy learning algorithm based on Inverse-Gap Weighting (IGW) that achieves $\widetilde{\mathcal O}(T^{3/4})$ regret over $T$ episodes.
Finally, in \pref{sec: experiment}, we evaluate our approach on both a synthetic episodic MDP benchmark and a real user booking dataset.
The empirical results demonstrate that our method effectively learns personalized objectives from multi-turn interactions.

\subsection{Related Works}

\paragraph{Interaction-Grounded Learning and Learning from Implicit Feedback}
The study of Interaction-Grounded Learning was initiated by \citet{xie2021interaction}, who modeled feedback as being conditionally independent of both context and action given the underlying latent reward.
This framework was subsequently extended by \citet{xie2022interaction} to allow feedback to depend on the chosen action, while remaining independent of the context conditioned on the reward and action.
\citet{hu2024information} later adopted the original formulation of \citet{xie2021interaction}, introducing an information-theoretic mechanism to explicitly enforce the required conditional independence.
Addressing the need for personalization, \citet{maghakian2023personalized} empirically extended the framework to incorporate user-specific reward structures, while \citet{zhang2024provably} developed the first computationally efficient algorithm with provable sublinear regret in this personalized setting.
More recently, \citet{xu2025provably} investigated learning with natural language feedback, establishing regret guarantees assuming access to a reliable signal-verification oracle. However, to the best of our knowledge, no prior work addresses the challenge of interaction-grounded learning with personalized feedback in the multi-step MDP setting.

\paragraph{Contextual Bandits and MDPs}
Our work also connects to a broad literature on statistically and computationally efficient contextual learning.
This includes classical contextual bandits \citep{langford2007epoch,agarwal2012contextual,agarwal2014taming,foster2018contextual,foster2018practical,foster2020beyond,simchi2021bypassing}, where the learner directly observes the reward for the selected action.
Extensions of this standard model consider structured observations, such as contextual bandits with feedback graphs \citep{zhang2024efficient,zhang2024practical}, where reward visibility is governed by an underlying graph structure, and contextual partial monitoring \citep{bartok2012partial,kirschner20a}, where feedback is provided via a signal matrix or linear operator.
Most relevant to our setting is the recent line of work on contextual MDPs \citep{levy2022learning,levy2023optimism,levy2023efficient,levy2024eluderbased}, which studies learning with general function approximation under stochastic or adversarial contexts. However, unlike our work, all these approaches assume access to direct reward observations.
\section{Preliminaries}\label{sec: prelim}
Throughout this paper, we use $[N]$ to denote the set $\{1,2,\dots, N\}$ for any positive integer $N$, and $\Delta_{\calB}$ to denote the probability simplex over a finite set $\calB$.

We study IGL with personalized feedback within the framework of contextual episodic MDPs with a finite horizon $H\geq 1$ and an unknown transition kernel.
Formally, the MDP is defined by a finite state space $\calS$, a finite action set $\calA$ with cardinality $|\calA|=K$, and a transition kernel $\calP:\calS\times\calA\mapsto\Delta_{\calS}$, where $\calP(s'|s,a)$ denotes the probability of transitioning to state $s'$ after taking action $a$ in state $s$.\footnote{For notational simplicity, we assume a common action set $\calA$ across all states; however, our results extend straightforwardly to the case of state-dependent action sets.}
Without loss of generality, we assume the state space $\calS$ is partitioned into $H$ disjoint layers, denoted as $\calS = \bigcup_{h=1}^H \calS_h$, where $\calS_1=\{s_1\}$ consists of a fixed singleton start state.
Transitions occur solely between consecutive layers, meaning that $\calP(s'|s,a) > 0$ implies $s\in \calS_h$ and $s'\in \calS_{h+1}$ for some $h \in [H-1]$.

\paragraph{Learning Protocol.}
The interaction proceeds in episodes $t \in [T]$.
At the beginning of each episode $t$, the learner observes a context $x_t \in \calX$ drawn i.i.d.\@ from an unknown distribution $\calD$. Define 
The learner then selects a policy $\pi_t: \calX \times \calS \to \Delta_{\calA}$ based on the current context and history.
Starting from the initial state $s_{t,1} = s_1$, for each step $h \in [H]$, the learner observes the current state $s_{t,h} \in \calS_h$, samples an action $a_{t,h} \sim \pi_t(\cdot| x_t, s_{t,h})$, and transitions to the next state $s_{t,h+1} \sim \calP(\cdot|s_{t,h}, a_{t,h})$. Note that we assume the transition dynamics are {independent} of the context.
This modeling choice is applicable to a wide range of scenarios, including multi-task reinforcement learning, where the context defines the task while the underlying physical/logical laws of the environment remain invariant across tasks.

Crucially, unlike classic contextual MDPs, the learner does not observe explicit rewards during the episode.
Instead, after the trajectory $\tau_t = (s_{t,1}, a_{t,1}, \ldots, s_{t,H}, a_{t,H})$ concludes, the learner receives a feedback signal $y_t \in \calY$ for some observation space $\calY$.
This signal serves as a proxy for an unobserved binary reward $r_t \in \{0,1\}$, which is realized solely at the final step, i.e., $r_t \sim \text{Bernoulli}(f^\star(x_t, s_{t,H}, a_{t,H}))$ for some fixed and unknown reward function $f^\star$.
We focus on this \emph{terminal-reward} MDP setting as it naturally models applications like LLM interactions, where user satisfaction depends primarily on the quality of the final response rather than the intermediate conversation leading up to it.
This setup extends the contextual bandit framework ($H=1$) studied in \citet{zhang2024provably} to the setting with general horizon $H$.

\paragraph{Feedback Structure.}
Since the true reward $r_t$ is latent, learning is impossible without structural assumptions linking the observed feedback $y_t$ to $r_t$.
Extending the standard IGL framework with personalized feedback~\citep{maghakian2023personalized,zhang2024provably}, we assume the feedback mechanism satisfies the following conditional independence property, which says that the feedback is conditionally independent of the trajectory given the context, the final state, and the realized reward.

\begin{assumption}[Conditional Independence of Feedback]\label{asm:feedback}
For any tuple of context, trajectory, reward, and feedback $(x, \tau, r, y)$ where $\tau=(s_1,a_1,\dots,s_{H},a_H)$, the feedback $y$ is conditionally independent of the trajectory prefix $(s_1, a_1, \ldots, s_{H-1}, a_{H-1})$ and the action $a_H$, given the context $x$, the final state $s_H$, and the realized reward $r$.
Formally,
$
    y \perp (s_1, a_1, \ldots, s_{H-1}, a_{H-1}, a_H) \mid (x, s_H, r).
$
\end{assumption}

This assumption captures the intuition that user feedback (e.g., a ``thumbs up'') is a reaction to the final output quality (the realized reward) relative to the user's intent, rather than the specific means used to achieve it.
Consequently, conditioned on the reward, the feedback is independent of both the internal reasoning path and the final action $a_H$, as any action yielding the same reward is indistinguishable from the perspective of the feedback mechanism. 

\paragraph{Realizability.} Following the literature of realizable contextual MDP \citep{levy2023efficient,levy2023optimism,levy2024eluderbased} and context-dependent IGL \citep{zhang2024provably}, we make the following realizability assumption on the expected reward given the context, state, and action, as well as a decoder for the realized reward.
\begin{assumption}[Realizability]\label{asm:realize}
The learner is given a function class
$\calF \subseteq [0,1]^{\calX\times\calS_H\times\calA}$
and, for each $s\in\calS_H$, a state-specific decoder class
$\Phi_s \subseteq \{0,1\}^{\calX\times\calY}$.
Define the product class $\Phi \triangleq \prod_{s\in\calS_H}\Phi_s$.
Each $\phi\in\Phi$ is a tuple $(\phi_s)_{s\in\calS_H}$ and induces a mapping
$\phi:\calX\times\calY\times\calS_H\to\{0,1\}$ via $\phi(x,y,s)\triangleq \phi_s(x,y)$.
We assume that there exist unknown $f^\star\in\calF$ and $\phi^\star=(\phi_s^\star)_{s\in \calS_H}\in\Phi$ such that for all $t\in[T]$,
$\E[r_t|x_t,s_{t,H},a_{t,H}]=f^\star(x_t,s_{t,H},a_{t,H})$
and $\phi^\star(x_t,y_t,s_{t,H})=r_t$.
\end{assumption}

Following \citep{zhang2024provably}, we assume that $\calF$ and $\Phi$ are both finite function classes and denote $|\calF|$ and $|\Phi|$ as the cardinality of $\calF$ and $\Phi$. Our results can be directly extended to broader function classes with infinite size, which will be discussed in later sections.

\paragraph{Identifiability for States.} 
As pointed out by \citet{xie2021interaction,maghakian2023personalized,xie2022interaction,zhang2024provably}, it is impossible to learn the optimal policy without breaking the symmetry between reward values of $0$ and $1$. To address this, we introduce the following assumption on the structure of the reward function:

\begin{assumption}\label{asm:identify}
    We assume that the true expected reward function $f^\star$ defined in \pref{asm:realize} satisfies that: there exist known constants $M\in(0,\frac{K}{2})$, $c\in[0,1]$, and $\theta \in (0,1)$ such that given context $x$, each final-layer state $s \in \calS_H$ falls into either one of the following two categories:
    \begin{itemize}
        \item \textbf{Heterogeneous state}: $\sum_{a\in[K]} f^\star(x,s,a) \leq M < \frac{K}{2}$, $\max_{a\in[K]} f^\star(x,s,a) \geq \theta > 0$, with $\frac{\theta(K-M)}{M} \triangleq\sigma>1$.
        \item \textbf{Homogeneous state}: for all $a \in [K]$, $f^\star(x,s,a) = c$.
    \end{itemize}
\end{assumption}

This assumption generalizes the identifiability condition from the contextual bandit setting ($H=1$) of \citet{zhang2024provably} to the episodic MDP setting.
Specifically, the {heterogeneous} state satisfies the same sparsity and separation conditions as in \citet{zhang2024provably}, ensuring sufficient information asymmetry to enable efficient learning.
In contrast, the {homogeneous} state captures MDP-specific scenarios where actions become indistinguishable due to identical feedback.
Note that these two definitions are mutually exclusive: for a homogeneous state, the total reward is $\sum_a f^\star(x,s,a) = Kc$ and the maximum is $\max_a f^\star(x,s,a) = c$, which implies the separation ratio $\frac{\theta(K-M)}{M} \le \frac{c(K-Kc)}{Kc} < 1$, strictly violating the heterogeneity condition $\sigma > 1$.
The existence of homogeneous state is common in scenarios such as learning language model where the system may become stuck in ``degenerate'' states, generating uniformly poor responses regardless of the chosen actions. 

One may also question why the constant $c$ needs to be known for homogeneous states. However, we argue that this knowledge is necessary for learning the optimal policy because, within an homogeneous state, the feedback is identical across all actions, making the reward value indistinguishable from interaction data alone. Without knowing $c$, the learner cannot determine the relative value of reaching an homogeneous state versus a heterogeneous state, leading to an impossibility to learn the optimal policy. Practically, homogeneous states typically correspond to degenerate scenarios or failure modes, such as a conversational agent producing a generic ``I don't know'' response, where the designer can explicitly assign a default utility (e.g., $c=0$) to penalize uninformative outcomes.

\paragraph{Goal.} The goal of the learner is to design a sequence of policies $\{\pi_t\}_{t=1}^T$ that maximize the cumulative reward over $T$ episodes. We measure the performance via the notion of regret defined as follows:
\begin{align}\label{eqn:regret}
    \Reg \triangleq T\cdot V^\star - \E\left[\sum_{t=1}^TV(\pi_t)\right].
\end{align}
Here, $V(\pi) \triangleq \E [ f^\star(x, s_{H}, a_{H}) ]$ denotes the value of policy $\pi$, 
where the expectation is taken over the random context and the trajectory $\{(s_{h}, a_{h})\}_{h=1}^H$ induced by the policy $\pi$ given the context $x$ and the transition dynamics. $V^\star$ denotes the highest possible value achieved by any stochastic policy in the class $\Pi \triangleq \{ \pi : \calX \times \calS \to \Delta_{\calA} \}$, and we denote the corresponding optimal policy by $\pi^\star \triangleq \argmax_{\pi \in \Pi} V(\pi)$. The expectation in \pref{eqn:regret} is taken over the internal randomness of the algorithm and the stochasticity of the algorithm's observations.

\paragraph{Other Notations.}  For all $s\in \calS$, define $P^{\pi}(s)$ as the probability of reaching state $s$ under policy $\pi$ and define $p_s^\star = \max_{\pi\in \Pi}P^{\pi}(s)$ to be the maximum probability of reaching state $s$. Denote $\mathbf{1}$ to be the all-one vector in an appropriate dimension. Denote $q_{\unif}$ to be the uniform distribution over an appropriate action space. 

\section{Reward Decoder Learning}\label{sec: decoder}
In this section, we detail the construction of our reward estimator, which transforms indirect feedback into a usable proxy for latent rewards.
Our approach proceeds in three stages:
first, we identify reachable terminal states to ensure valid data collection;
second, we learn the inverse kinematic (posterior distribution) of actions at these states;
and third, we construct a Lipschitz-continuous reward estimator that robustly handles both heterogeneous and homogeneous states.

\subsection{Step 1: Reachable State Identification}\label{sec: reachability}
The first challenge in learning the reward decoder is ensuring we can collect sufficient feedback data.
If a state has intrinsically low maximum reaching probability $p_s^\star$, it will never be visited often enough to allow reliable estimation; at the same time, its contribution to the total regret is also inherently negligible.
Therefore, our goal in this step is to identify the set of states that are reachable with sufficiently high probability. 

We begin by learning a homing policy for each state $s\in\calS_H$ to maximize visitation.
Specifically, for each target state $s$, we define a dummy reward function that is $1$ if the learner reaches $s$ and $0$ otherwise. While any RL algorithm that provides provable PAC or regret guarantees suffices, we employ \euler~\citep{EULER2019} on this auxiliary MDP to learn a policy $\wh{\pi}_s$.
This ensures that after $\otil(1/\epsilon^2)$ episodes, we obtain a policy capable of visiting $s$ near-optimally (up to error $\epsilon$).
The full procedure is outlined in \pref{alg:homing} and deferred to \pref{app: decoder}.

\begin{restatable}{lemma}{stateIdentify}\label{lem:state_identify}
    Applying \pref{alg:homing} to each $s\in\calS_H$ with $N=\frac{C\cdot SKH\log(SKH/\delta)}{\eps^2}$ episodes guarantees that with probability at least $1-\delta$, the output policy set $\{\wh{\pi}_s\}_{s\in\calS_H}$ satisfies $P^{\wh{\pi}_s}(s)\geq p_s^\star - \eps$ for all $s\in\calS_H$, where $C\geq 1$ is a universal constant.
\end{restatable}

Next, we quantify the reachability of each state using these learned policies.
Since both the optimal reachability $p_s^\star$ and $P^{\wh{\pi}_s}(s)$ are unknown, we must estimate them based on empirical data.
Let $\wh{p}_s$ denote the empirical visitation frequency of state $s$ observed by running $\wh{\pi}_s$ for $N$ episodes.
A standard Hoeffding-Azuma inequality guarantees that with probability at least $1-\delta$, $|\wh{p}_s-P^{\wh{\pi}_s}(s)|\leq \sqrt{\frac{\log(S/\delta)}{2N}}$ for all $s\in \calS_H$.
Combining this concentration bound with \pref{lem:state_identify} allows us to reliably distinguish between states with high and low maximum reaching probability based solely on the observable empirical visitation frequency $\wh{p}_s$.

\begin{restatable}{lemma}{identifySignificant}\label{lem:identify_significant_state}
    Let $\beta=\sqrt{\frac{\log(SKH/\delta)}{2N}}$. Then, with probability at least $1-2\delta$, for any threshold $\tau\in(0,1)$:
    (i) if $\wh{p}_s\ge \tau+\beta$, then $p_s^\star \geq P^{\wh{\pi}_s}(s)\ge \tau$;
    (ii) if $\wh{p}_s\le \tau-(\beta+\eps)$, then $p_s^\star\le \tau$.
\end{restatable}

Finally, based on these estimates, we filter the state space to restrict our attention to states that can be reached often enough for reliable learning.
We define the set of reliably reachable states as:
\begin{align}\label{eqn:significant_state}
    \wh{\calS}_{H,\eps} \triangleq \{s\in\calS_H : \wh{p}_s \ge 4\eps\}.
\end{align}
By applying \pref{lem:identify_significant_state} with the chosen sample size $N$ defined in \pref{lem:state_identify} and threshold parameters, we can guarantee that any state excluded from this set (i.e., $s\notin \wh{\calS}_{H,\eps}$) has a maximum reaching probability $p_s^\star \le 6\eps$.
Consequently, omitting these states incurs negligible regret, allowing us to focus the subsequent reward decoding steps exclusively on $\wh{\calS}_{H,\eps}$.

\subsection{Step 2: Inverse Kinematic Learning}\label{sec: ik_learning}
Having identified the set of reliably reachable states $\wh{\calS}_{H,\eps}$, we proceed to the second component: inverse kinematic learning. Specifically, inspired by~\citet{zhang2024provably}, we adopt an \emph{inverse kinematic} strategy, which infers the latent reward by analyzing the posterior distribution of actions conditioned on the context, state, and observed feedback.

We first characterize the theoretical properties of this posterior when actions are selected uniformly.

\begin{restatable}{lemma}{posterior}\label{lem:posterior}
For any context $x\in\calX$ and state $s\in\calS_H$, suppose the action $a$ is chosen uniformly from $[K]$.
Let $r$ denote the realized (binary) reward and $y$ denote the feedback.
Then the posterior distribution of $a$ conditioned on $(x,y,s)$ is given by
\begin{align*}
\Pr[a | x,y,s]
&= \frac{f^\star(x,s,a)\cdot \phi^\star(x,y,s)}{\sum_{a'=1}^K f^\star(x,s,a')} + \frac{(1-f^\star(x,s,a))\cdot (1-\phi^\star(x,y,s))}{K-\sum_{a'=1}^K f^\star(x,s,a')}.
\end{align*}
\end{restatable}

The following lemma establishes how the true reward can be recovered from this posterior.
Crucially, for heterogeneous states, the posterior deviates from uniformity in the $\ell_\infty$-norm, enabling reward inference.
In contrast, for homogeneous states, the posterior collapses to the uniform distribution, signalling that the expected reward is the known constant $c$.

\begin{restatable}{lemma}{ikbinary}\label{lem:decoder}
For any context $x\in\calX$, action $a\in[K]$, and state $s\in\calS_H$, define $h^\star(x,y,s)\in\Delta_K$ with components
\begin{align}
h_a^\star(x,y,s)
&\triangleq \frac{f^\star(x,s,a)\cdot \phi^\star(x,y,s)}{\sum_{a'=1}^K f^\star(x,s,a')} + \frac{(1-f^\star(x,s,a))(1-\phi^\star(x,y,s))}{K-\sum_{a'=1}^K f^\star(x,s,a')}. \label{eqn:ik_h}
\end{align}
If $s$ is a heterogeneous state, then:
\begin{itemize}
    \item $\bigl\|h^\star(x,y,s)-\tfrac{1}{K}\mathbf{1}\bigr\|_\infty \ge \kappa$, where $\kappa\triangleq\frac{K\theta-M}{K(K-M)}$;
    \item $r(x,s,a)=\phi^\star(x,y,s)\ge \mathbbm{1}\!\left\{h_a^\star(x,y,s)\ge \frac{\theta}{M}\right\}$ for all $a\in[K]$;
    \item $r(x,s,a)=\phi^\star(x,y,s)=\mathbbm{1}\!\left\{h_a^\star(x,y,s)\ge \frac{\theta}{M}\right\}$ when $a=\argmax_{a'\in[K]}f^\star(x,s,a')$;
\end{itemize}
otherwise, $h^\star(x,y,s)=\tfrac{1}{K}\mathbf{1}$.
\end{restatable}
The proof is provided in \pref{app: decoder}.
This result generalizes the inverse-kinematics characterization of Lemma 2 in \citet{zhang2024provably} to the MDP setting. Specifically, when $s$ is a heterogeneous state, $\mathbbm{1}\!\left\{h_a^\star(x,y,s)\ge \frac{\theta}{M}\right\}$ serves as an \emph{underestimator} of the true reward $r(x,s,a)$ and matches the
reward of the optimal policy. This matches the main idea in the estimator construction in the single-state case shown in \citet{zhang2024provably}. However, unlike the single-state setting, here the posterior serves an additional purpose, which acts as a discriminator between state types.
If the posterior is bounded away from uniformity, the state is identified as heterogeneous; conversely, if it is uniform, the state is then homogeneous.

With the theoretical characterization of $h^\star$ established, we now turn to estimating it from data.
We focus exclusively on the set of reachable states $\wh{\calS}_{H,\eps}$ identified in Step 1.
For each $s \in \wh{\calS}_{H,\eps}$, we employ the approximate homing policy $\wh{\pi}_s$ learned in Step 1 to collect a dataset of context--action--signal tuples.
Whenever $\wh{\pi}_s$ reaches state $s$, we record the tuple $(x,s,a,y)$. Note that according to the definition of the policies collected in \pref{alg:homing}, action $a$ is uniformly sampled from $[K]$. Since every $s \in \wh{\calS}_{H,\eps}$ satisfies $\wh{p}_s \ge 4\eps$, \pref{lem:identify_significant_state} ensures the true reaching probability is at least $3\eps$.
This guarantees that we can collect a dataset $D_s$ of size $N_0$ efficiently by using $\otil(\frac{N_0}{\eps})$ number of episodes.
The procedure is shown in \pref{alg:tuple_collection} deferred to \pref{app: decoder}.

\begin{restatable}{lemma}{sampling}\label{lem:sampling}
With probability at least $1-4\delta$, \pref{alg:tuple_collection} terminates after running
$\order\!\left(\frac{S\,(N_0+\log(S/\delta))}{\eps}\right)$ episodes.
\end{restatable}

Finally, we estimate the posterior $h^\star$ for each $s \in \wh{\calS}_{H,\eps}$ using Empirical Risk Minimization (ERM) on the collected dataset $D_s$.
We consider the hypothesis class $\calH_s$ defined as:
\begin{align}\label{eqn:H_class}
\calH_s
\triangleq\Bigl\{
h_s:\calX\times\calY\to\Delta_{[K]}
~\Big|~
h_{s,a}(x,y)
=
\frac{f(x,a)\phi_s(x,y)}{\sum_{i=1}^K f(x,i)}+
\frac{(1-f(x,a))(1-\phi_s(x,y))}{K-\sum_{i=1}^K f(x,i)},f\in\calF,\phi_s\in\Phi_s,a\in[K]
\Bigr\}.
\end{align}
Since the true posterior $h_s^\star \in \calH_s$ for all $s\in\calS_H$ based on \pref{eqn:ik_h}, we can learn the posterior distribution at each state using an empirical risk minimization (ERM) oracle with squared loss. Specifically, for each state $s\in\wh{\calS}_{H,\eps}$, we estimate $h_s^\star$ using $D_s$, and the following lemma, adapted from Lemma~3 of~\citet{zhang2024provably}, shows that the ERM estimator achieves an excess risk of $\order({\log(|\calH|)/|D_s|})$ with high probability.

\begin{restatable}[Lemma 3 of \citet{zhang2024provably}]{lemma}{erm}\label{lem:erm}
For each $s\in\wh{\calS}_{H,\eps}$, let $\wh{h}_s$ be the ERM estimator defined as:
\begin{align}\label{eq:erm_h}
\wh{h}_s
\triangleq
\argmin_{h_s\in\calH_s}
\sum_{(x_i, s, a_i, y_i) \in D_s}
\bigl\|h_s(x_i,y_i)-e_{a_i}\bigr\|_2^2,
\end{align}
where $\{D_s\}_{s\in\wh{\calS}_{H,\eps}}$ is collected by running \pref{alg:tuple_collection} with $N_0$.
Then, with probability at least $1-\delta$, for all $s\in\wh{\calS}_{H,\eps}$,
\begin{align*}
\E\!\left[\bigl\|\wh{h}_s(x,y)-h^\star(x,y,s)\bigr\|_2^2\right]
\le
\order\!\left({\frac{\log\!\left(S|\calH|/\delta\right)}{N_0}}\right),
\end{align*}
where the expectation is taken over $x\sim\calD$, $a\sim q_{\unif}$, and the resulting feedback $y$ conditioned on $(x,s,a)$.
\end{restatable}

For notation convenience, we define the vector of estimators as $\wh{h}\triangleq(\wh{h}_s)_{s\in\wh{\calS}_{H,\eps}}$, and denote $\wh{h}(x,y,s)=\wh{h}_s(x,y)$ for any $s\in\wh{\calS}_{H,\eps}$.

\subsection{Step 3: Lipschitz Reward Estimator Construction}\label{sec: lipschitz_estimator}
In the final step, we construct a robust reward estimator using the learned posterior $\wh{h}$ from Step 2.
Based on the characterization in \pref{lem:decoder}, a natural choice for the reward estimator would be the plug-in indicator function:
\begin{align*}
\wh{r}(x,y,s,a) = c\cdot \mathbbm{1}\!\left\{\left\|\wh{h}(x,y,s)-\tfrac{1}{K}\mathbf{1}\right\|_{\infty}\leq \nu\right\}+\mathbbm{1}\!\left\{\bigl\|\wh{h}(x,y,s)-\tfrac{1}{K}\mathbf{1}\bigr\|_2>\nu\right\}
\,\mathbbm{1}\!\left\{\wh{h}_a(x,y,s)\ge \tfrac{\theta}{M}\right\},
\end{align*}
for certain small positive value $\nu$.
Intuitively, this estimator assigns the constant reward $c$ if the state is identified as homogeneous when the posterior distribution is close to uniform, and otherwise applies the thresholding rule from \pref{lem:decoder} to recover the reward.

However, as observed in \citet{zhang2024provably}, the indicator function is discontinuous and non-Lipschitz.
Consequently, small estimation errors in $\wh{h}$, which are inevitable even with the generalization guarantees of \pref{lem:erm}, can lead to huge errors in the predicted reward.
To address this, we generalize the Lipschitz surrogate construction of \citet{zhang2024provably} to our multi-step setting, ensuring a smooth transition between the homogeneous and heterogeneous state regimes.

\begin{figure}[t]
    \centering
    \includegraphics[width=\linewidth]{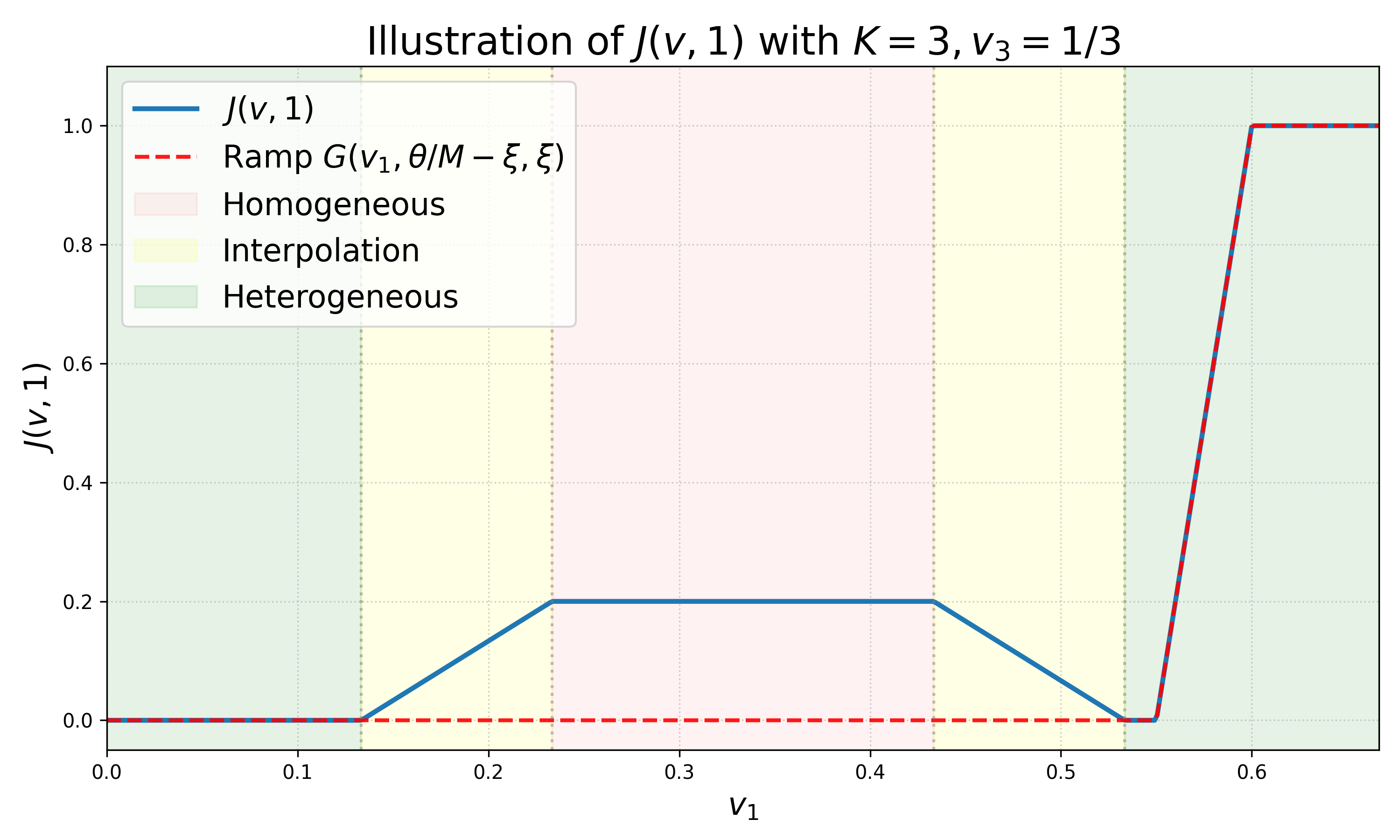}
    \caption{Illustration of Lipschitz reward decoder $J(v,1)$ when $v_1 \in [0, 2/3]$, $v_3=1/3$, $K=3$, $M=1$, $\theta=0.6$, $c=0.2$, and $\kappa=0.2$.}
    \label{fig:J_function_K3}
\end{figure}

Given any posterior distribution $v\in \Delta_{[K]}$, let $\Delta(v) \triangleq \|v-\frac{1}{K}\mathbf{1}\|_\infty$ denote its $\ell_\infty$-distance from uniformity.
We define the Lipschitz surrogate estimator $J(v,a)$ as follows:
\begin{align}\label{eqn:reward_decoder}
J(v,a) =\begin{cases}
c, & \Delta(v)\le \kappa/2,\\
\frac{2c(\kappa-\Delta(v))+(2\Delta(v)-\kappa)\cdot G\left(v_a,\tfrac{\theta}{M}-\xi,\xi\right)}{\kappa},
&\Delta(v)\in(\kappa/2,\kappa),\\
G\left(v_a,\tfrac{\theta}{M}-\xi,\xi\right),
& \text{otherwise},
\end{cases}
\end{align}
where the smoothing parameter $\xi$ is defined as $\xi = \frac{1}{2}(\frac{\theta}{M}-\frac{1}{K-M})$, and $G(\alpha,\beta,\lambda)$ is the ramp function defined in \citet{zhang2024provably}:
\begin{align*}
    G(\alpha,\beta,\lambda) \triangleq \frac{\alpha-\beta}{\lambda}\mathbbm{1}\{\beta\leq \alpha
    <\beta+\lambda\} + \mathbbm{1}\{\alpha\geq \beta+\lambda\}.
\end{align*}
To provide intuitions on the estimator defined in \pref{eqn:reward_decoder}, we visualize its geometry in \pref{fig:J_function_K3} and explain each part below:
\begin{itemize}
    \item \textbf{Homogeneous Region (Pink):} When the posterior is near-uniform ($\Delta(v) \le \kappa/2$), the state is identified as homogeneous. Here, the estimator is clamped to the constant $c$, represented by the flat section in the center of the figure.
    \item \textbf{Heterogeneous Region (Green):} When the posterior is distinct ($\Delta(v) \ge \kappa$), the state is identified as heterogeneous. In this region, the estimator follows the inverse-kinematics ramp function $G$ (shown as the red dashed line), which itself is a linear interpolation connecting the logical values 0 and 1.
    \item \textbf{Interpolation Region (Yellow):} The middle case in \pref{eqn:reward_decoder} (for $\Delta(v)\in(\kappa/2,\kappa)$) serves as a \emph{Lipschitz bridge}. As seen in the figure, this region linearly interpolates between the constant $c$ (from the pink region) and the value of $G$ (from the green region). This ensures that $J(v,a)$ smoothly transitions between the two logic types, preserving Lipschitz continuity essential for the regret analysis.
\end{itemize}
The following lemma formally verifies that this construction preserves the key properties of the ideal estimator while ensuring the required smoothness.

\begin{restatable}{lemma}{decoderProperty}\label{lem:ambigurous_state_identify}
For any context $x\in\calX$, state $s\in\calS_H$, action $a\in[K]$, and feedback $y$ generated by a realized reward $r(x,s,a)$, the estimator $J$ satisfies:
\begin{itemize}[leftmargin=*,nosep]
    \item If $s$ is a heterogeneous state:
    \begin{itemize}
        \item $r(x,s,a) = \phi^\star(x,y,s)\geq J(h^\star(x,y,s),a)$;
        \item $r(x,s,a) = \phi^\star(x,y,s)= J(h^\star(x,y,s),a)$ if $a=\argmax_{a'\in[K]}f^\star(x,s,a)$.
    \end{itemize}
    \item If $s$ is a homogeneous state: $\E[r(x,s,a)] = J(h^\star(x,y,s),a) = c$.
\end{itemize}
Moreover, $J(v,a)$ is $L$-Lipschitz in $v$ with respect to the $\ell_\infty$-norm, with $L=\frac{4}{\kappa}+\frac{1}{\xi}$, where $\kappa$ and $\xi$ are defined in \pref{lem:decoder} and \pref{eqn:reward_decoder} respectively.
\end{restatable}
The proof is deferred to \pref{app: decoder}. A crucial consequence of \pref{lem:ambigurous_state_identify} is that the Lipschitz continuity of $J(v,a)$ enables a stable transfer from the true posterior $h^\star$ to its empirical estimate $\wh{h}$, ensuring that $J(\wh{h}(x,y,s),a)$ will be close to the true estimator $J(h^\star(x,y,s),a)$. In the next section, we leverage this reward estimator to design learning algorithms with provable performance guarantees.

\section{Online Policy Learning}\label{sec: regret}
Building upon the reward estimator $J(\wh{h}(x,y,s),a)$, we propose our algorithm for learning the optimal policy.
Following the recent framework of realizable contextual learning~\citep{foster2020beyond,foster2021efficient}, we assume access to an online square-loss regression oracle $\AlgSq$: at round $t$, \AlgSq predicts $\wh{f}_t$ from the convex hull of $\calF$, and then receives a tuple $(x_t, s_{t,H}, a_{t,H}, \wt{r}_t)$, where $\wt{r}_t \triangleq J(\wh{h}(x_t, y_t, s_{t,H}), a_{t,H})$ serves as the estimated reward, and incurs a squared loss $(\wh{f}_t(x_t, s_{t,H}, a_{t,H}) - \wt{r}_t)^2$.
We assume that $\AlgSq$ satisfies certain regret bound with respect to squared-loss against the best predictor in $\calF$.
\begin{assumption}[Bounded squared-loss regret]
\label{asm:reward_oracle}
\AlgSq guarantees that
for any sequence $\{(x_t, s_{t,H},a_{t,H},\wt{r}_t)\}_{t\in[T]}$, we have the following holds for certain $\RegSq$:
\begin{align*}
    \sum_{t=1}^T (\wh{f}_t(x_t, s_{t,H}, a_{t,H}) - \tilde{r}_t)^2 -\inf_{f \in \mathcal{F}} \sum_{t=1}^T (f(x_t, s_{t,H}, a_{t,H}) - \tilde{r}_t)^2 \le \RegSq.
\end{align*}
\end{assumption}
Many standard function classes admit such an oracle. For instance, for a finite class $\calF$, Vovk's aggregation algorithm achieves $\RegSq = \order(\log |\calF|)$~\citep{vovk1995game}. We refer the readers to \citet{foster2020beyond} for more examples.
We also define the reward estimator based on the ground-truth inverse kinematics $\underline{f}^\star(x,s,a) \triangleq \E_{y|x,s,a}[J(h^\star(x,y,s), a)]$.
By \pref{lem:ambigurous_state_identify}, this function serves as a valid lower bound on the true reward $f^\star(x,s,a)$.
Following the assumption in \citet{zhang2024efficient}, we also assume that $\underline{f}^\star\in\calF$.
\begin{assumption}[Lower bound realizability]\label{asm:realizability_lower}
   Assume that $\underline{f}^\star\in\calF$ where 
   $\underline{f}^\star(x, s, a) \triangleq  \E_{y \mid x,s,a} [J(h^\star(x,y,s), a)] $.
\end{assumption}

We now present our policy learning algorithm, \pref{alg:online_mdp} (deferred to \pref{app: policy_learning_alg}), which builds upon the OMG-CMDP! algorithm proposed by \citet{levy2023efficient}.
Designed for realizable contextual MDPs with direct reward feedback, OMG-CMDP! operates by maintaining estimates for both the reward and transition models via online regression oracles. Specifically, at each round $t$, the algorithm utilizes the history up to $t-1$ to approximate the rewards via a squared-loss oracle and dynamics via a log-loss oracle.
Upon observing the current context $x_t$, it solves an optimization problem to find an occupancy measure that maximizes the estimated reward subject to the estimated dynamics, regularized by a log-barrier term. Finally, the learner executes the policy induced by this optimal occupancy measure.

Compared to OMG-CMDP!, our algorithm incorporates two key modifications tailored to our setting: (1) instead of using a general log-loss regression oracle to estimate the transition, since we assume that the transition is independent of the context, we employ a simpler \textit{smoothed empirical mean estimator} constructed from online visitation counts to estimate the transition dynamics (\pref{line: transition_estimate}). We show in \pref{app:transition_estimation} that this estimator also enjoys a $\otil(S^2K)$ log-loss regret guarantee satisfying the requirements of the analysis in \citet{levy2023efficient}. (2) Since the true reward is latent, we utilize the learned reward estimator $J(\wh{h}(x,y,s),a)$ to generate proxy rewards and use them as the input for \AlgSq. To ensure accurate reward estimation, we use the reachable set $\widehat{\mathcal{S}}_{H,\epsilon}$ to filter the input tuples for \AlgSq: if $s_{t,H} \in \widehat{\mathcal{S}}_{H,\epsilon}$, we decode the feedback to obtain a proxy reward $\tilde{r}_t = J(\wh{h}(x_t, y_t, s_{t,H}), a_{t,H})$ and use the tuple $(x_t, s_{t,H}, a_{t,H}, \tilde{r}_t)$ to update the oracle (\pref{line: eqn_proxy} and \pref{line: update_oracle}); otherwise, we discard these episodes for reward learning. Since the total probability of visiting states outside $\widehat{\mathcal{S}}_{H,\epsilon}$ is bounded by $\mathcal{O}(TS\epsilon)$, ignoring them incurs negligible regret.

Our final full algorithm is shown in \pref{alg:igl_mdp}, which combines decoder learning and online policy learning,
and our main theoretical result is the following regret guarantee.

\setcounter{AlgoLine}{0}
\begin{algorithm}[t]
\caption{IGL for Contextual MDPs}
\label{alg:igl_mdp}
   \nl \textbf{Input:} function classes $\mathcal{F}$, $\Phi$, $\gamma > 0$, $\epsilon>0$, $N_0>0$, online regression oracle $\AlgSq$.

   \nl \For{each $s\in\calS_H$}{
        \nl Run \pref{alg:homing} with input $\delta = \frac{1}{T^2}$, $N$ defined in \pref{lem:ambigurous_state_identify}, and $s$, and output $\wh{\pi}_s$.
   }
   \nl Construct $\wh{\calS}_{H,\epsilon}$ defined in \pref{eqn:significant_state} using $\{\wh{\pi}_s\}_{s\in\calS_H}$.

   \nl Run \pref{alg:tuple_collection} with input $\wh{\calS}_{H,\epsilon}$, $\{\wh{\pi}_s\}_{s\in\wh{\calS}_{H,\epsilon}}$, and $N_0$, and obtain $\{D_s\}_{s\in\wh{\calS}_{H,\epsilon}}$.

   \nl Compute $\wh{h}_s$ for each $s\in\calS_{H,\eps}$ according to \pref{eq:erm_h} using $D_s$. Let $\wh{h}=(\wh{h}_s)_{s\in\wh{\calS}_{H,\epsilon}}$.

   \nl Run \pref{alg:online_mdp} with input $\wh{\calS}_{H,\eps}$, $\wh{h}$, $\calF$, $\gamma$, and \AlgSq for the remaining episodes.
   
\end{algorithm}

\begin{theorem}\label{thm:final_regret}
Under Assumptions \ref{asm:feedback}, \ref{asm:realize}, \ref{asm:identify}, \ref{asm:reward_oracle}, and \ref{asm:realizability_lower}, 
\pref{alg:igl_mdp} with certain choices of $\eps$, $\gamma$, and $N_0$ guarantees 
\begin{align*}
    \Reg=\otil\left(T^{\frac{3}{4}}S^{\frac{1}{2}}K^{\frac{1}{4}}H^{\frac{1}{4}}L^{\frac{1}{2}}+\sqrt{TSKH\RegSq}\right),
\end{align*}
where $L$ is defined in \pref{lem:ambigurous_state_identify}.
\end{theorem}

To our knowledge, this is the first sublinear regret guarantee for contextual MDPs with implicit personalized feedback. Our bound of $\otil(T^{3/4}+\sqrt{T\RegSq})$ is looser than the single-step result of $\otil(T^{2/3}+\sqrt{T\RegSq})$ proven in \citet{zhang2024provably}. This gap stems from the additional exploration required to learn the homing policy and collect sufficient samples from each state, which is not necessary in the single-step setting. Compared to \citet{levy2023efficient} where direct reward feedback is available, our approach incurs an additional $\otil(T^{3/4})$ overhead to handle the signal decoding. We leave the question of what the minimax regret bound is as an interesting future direction.

The full proof of this theorem is deferred to \pref{app:regret_analysis}. 
The primary analytical challenge lies in managing {reward misspecification}: our algorithm receives a learned proxy reward $J(\wh{h}(x,y,s),a)$ rather than the true underlying reward. We address this by leveraging the Lipschitz continuity of $J$, allowing us to bound the misspecification via the estimation error between the learned posterior $\wh{h}(x,y,s)$ and the true posterior $h^\star(x,y,s)$. This error is then controlled by the generalization guarantees of $\wh{h}$, which rely on the dataset providing sufficient coverage of the action space.

\section{Experiment}\label{sec: experiment}

\begin{figure}[t]
    \centering
    \begin{subfigure}{0.48\linewidth}
        \centering
        \includegraphics[width=\linewidth]{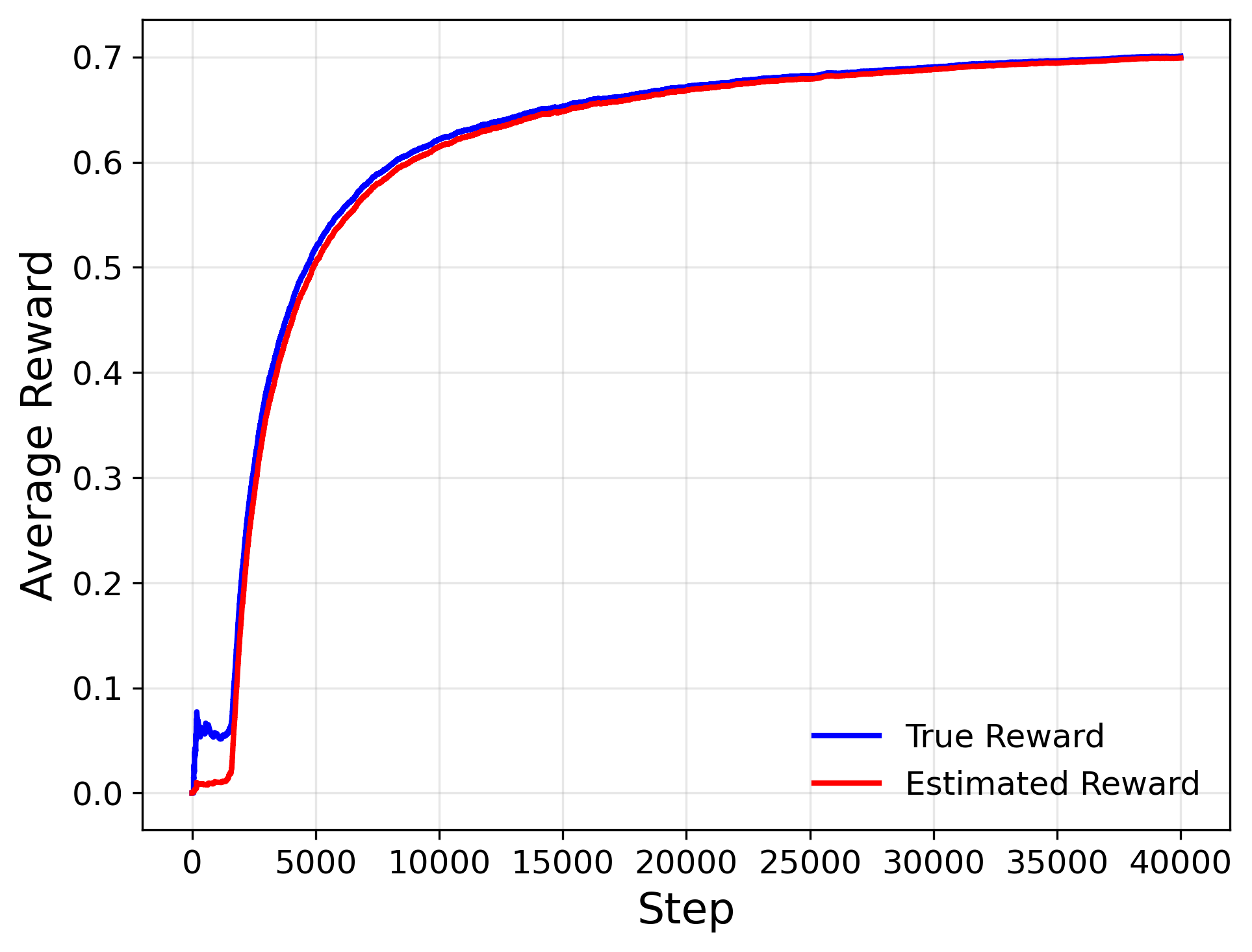}
        \caption{Synthetic MDP}
        \label{fig:reward}
    \end{subfigure}
    \hfill
    \begin{subfigure}{0.48\linewidth}
        \centering
        \includegraphics[width=\linewidth]{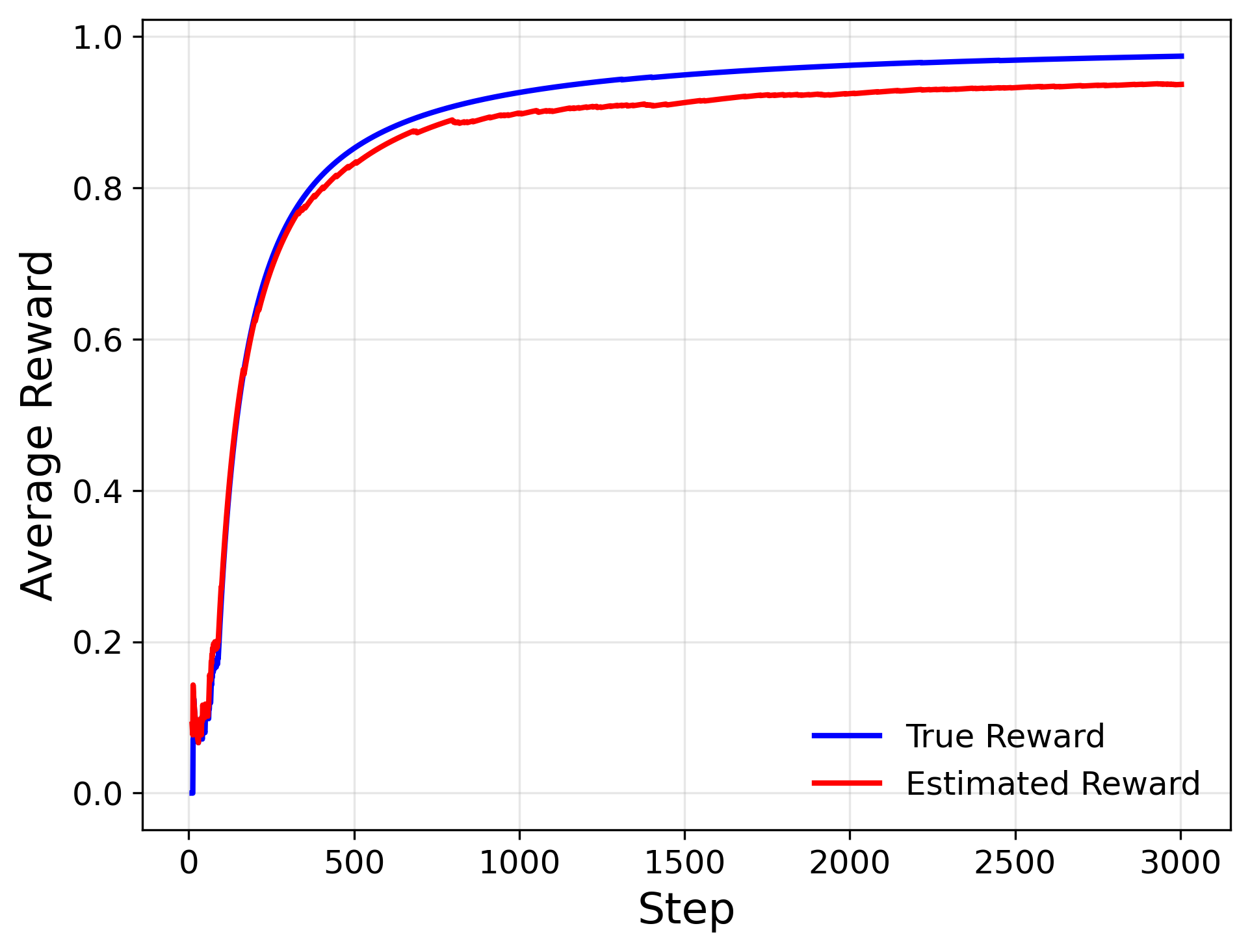}
        \caption{Real booking dataset}
        \label{fig:exp_real}
    \end{subfigure}
    \caption{Average reward during policy learning on synthetic and real datasets.}
    \label{fig:real_reward}
\end{figure}
 
\subsection{Experiments on Synthetic Datasets}\label{sec: synthe}

\paragraph{Dataset Construction.}
We construct a synthetic contextual episodic MDP with horizon $H=3$ and a binary context space $x_t \in \{\texttt{True}, \texttt{False}\}$.
At the start of each episode, the context is drawn with probabilities $\Pr(x_t=\texttt{True})=0.7$ and $\Pr(x_t=\texttt{False})=0.3$.
The state space is partitioned into layers $\calS_1=\{s_{1,g}\}$, $\calS_2=\{s_{2,g}, s_{2,b}\}$, and $\calS_3=\{s_{3,g}, s_{3,b}\}$, where indices $g$ and $b$ represent ``good'' and ``bad'' states, respectively.
The action set is $\calA=\{a_1, \dots, a_5\}$.
The dynamics and rewards are defined as follows:
\begin{itemize}[nosep, leftmargin=*]
    \item \textbf{Transitions ($h < 3$):} Transitions depend on the current state quality.
    From a good state $s_{h,g}$, action $a_1$ maintains the good state ($s_{h+1, g}$) with high probability $1-p$ and transitions to the bad state ($s_{h+1, b}$) with probability $p$.
    Actions $a_2, \dots, a_5$ have the reverse effect, transitioning to the bad state with probability $1-p$ and to the good state with probability $p$.
    From a bad state $s_{h,b}$, all actions deterministically transition to the next bad state $s_{h+1, b}$.
    \item \textbf{Rewards ($h=3$):} Rewards are binary and realized only at the final step.
    At the good terminal state $s_{3,g}$, action $a_1$ yields reward $1$ with probability $1-p_{\mathrm{reward}}$, while $a_2, \dots, a_5$ yield reward $1$ with probability $p_{\mathrm{reward}}$.
    At the bad terminal state $s_{3,b}$, the reward is deterministically $0$.
    \item \textbf{Feedback Signal:} The learner observes a flipped version of the reward depending on the context.
    Specifically, $y_t = r_t$ if $x_t=\texttt{True}$, and $y_t = 1-r_t$ if $x_t=\texttt{False}$.
\end{itemize}
In our experiments, we set parameters $p=p_{\mathrm{reward}}=0.1$.
Under this configuration, for either context, $s_{3,g}$ is a {heterogeneous state} with $\theta=0.9$ and $M=1.3$, while $s_{3,b}$ is a {homogeneous state} with constant reward $c=0$. More implementation details are deferred to \pref{app: synthe}.

\paragraph{Results} Our empirical results are summarized in \pref{fig:reward}, which plots the averaged cumulative reward over $T=40,000$ episodes of policy learning.
The blue curve depicts the performance evaluated against the unobserved true reward, while the red curve shows the performance as measured by our estimated reward decoder.
Two observations are as follows.
First, the decoded reward stays close to the true reward but remains always slightly lower.
This confirms the result in \pref{lem:decoder} that our estimator acts as a reward under-estimator.
Furthermore, the difference between the two curves vanishes at the end of training, showing that our decoder correctly recovers the true reward once the policy becomes optimal.
Second, as shown in \pref{fig:reward}, our algorithm steadily improves, approaching a final reward of $0.7$ (the optimal policy achieves reward $0.729$).
This consistent increase demonstrates that our method successfully optimizes the policy using only the decoded signal.

\subsection{Experiments on Real Datasets}\label{sec: real}
\paragraph{Dataset Construction.}
We constructed our dataset by merging two task-oriented dialogue corpora: the \texttt{dialog\_bAbI} dataset~\citep{bordes2016learning} for restaurant booking and a movie booking dialogue corpus~\citep{shah2018building}. The combined dataset consists of approximately 3,700 samples. We construct two-turn dialogues comprising a clarification query and a final booking action. Besides the optimal ground-truth actions, we employ an LLM to synthesize incorrect candidates including queries lacking sufficient detail and bookings containing inaccurate information. Generation details and the implementation details are deferred to~\pref{app: real_data} and~\pref{app: imple_real} respectively.

\paragraph{Results.} The performance on the real dataset is shown in~\pref{fig:exp_real}.
Note that the optimal reward is always $1$, corresponding to the correct booking action.
We observe that the average reward increases steadily during the training process, eventually reaching 0.95, which shows the effectiveness of our algorithm.
Similar to \pref{sec: synthe}, the estimator consistently serves as a lower bound for the ground truth reward, aligning well with our theoretical analysis.
\section{Conclusion and Future Works}\label{sec:conclusion}

In this paper, we studied Interaction-Grounded Learning with personalized feedback in contextual episodic MDPs with terminal rewards.
We proposed a two-stage algorithm that first learns a reward decoder via inverse kinematics, handling both significant and ambiguous states, and subsequently learns the optimal policy using these decoded signals.
We established theoretical regret guarantees, and our experiments on both synthetic and real-world data validate that our decoder accurately recovers ground-truth rewards, enabling near-optimal policy learning without direct supervision.
Future work includes extending our framework to context-dependent dynamics and generalizing the approach to handle intermediate personalized feedback.

\bibliography{ref}
\bibliographystyle{plainnat}

\newpage
\appendix

\section{Omitted Details in \pref{sec: decoder}}\label{app: decoder}
In this section, we present the omitted details in \pref{sec: decoder}, including omitted algorithm descriptions in \pref{app: decoder_alg} and omitted proofs in \pref{app: decoder_proof}.

\subsection{Omitted Algorithms in \pref{sec: decoder}}\label{app: decoder_alg}

We first show the pseudo code that computes a policy that approximates the homing policy to a given state $s\in S_H$ introduced in \pref{sec: reachability}.

\begin{algorithm}[h]
\caption{Homing policy learning for state $s$}
\label{alg:homing}
\textbf{Input:} failure probability $\delta\in(0,1)$, number of episodes $N$, target state $s\in\calS_H$.

Define a reward function $R:\calS_H\to\{0,1\}$ by $ R(s) \;=\; \mathbbm{1}\{s'=s\}$ for all $a\in \calA$ and $0$ otherwise.

Run \textsc{Euler} on the induced episodic MDP with reward function $R$,
horizon $H$, and failure probability $\delta$ for $N$ episodes.

Let $\Psi(s)=\{\pi^{(1)},\pi^{(2)},\ldots,\pi^{(N)}\}$ denote the set of policies executed by \textsc{Euler} during these episodes where we set the strategy of the final state $s$ to be $\pi^{(n)}(s)=\frac{1}{K}\mathbf{1}$ for all $n\in[N]$.

\textbf{Output.}
Output policy $\wh{\pi}_s$ defined as the uniform mixture over $\Psi(s)$.
\end{algorithm}

Next, we show the pseudo code for collecting context-action-signal tuples introduced in \pref{sec: ik_learning}.

\begin{algorithm}[h]
\caption{Tuple Collection}
\label{alg:tuple_collection}
\textbf{Input:} State set $\wh{\calS}_{H,\eps} $, policy set $\{\wh{\pi}_s\}_{s\in \wh{\calS}_{H,\eps} }$, sample size $N_0$.

\For{each $s\in \wh{\calS}_{H,\eps} $}{
    Initialize $n=0$ and dataset $D_s=\emptyset$.

    \While{$n \le N_0$}{
        Start a new episode and observe a context, denoted as $x_{s,n}$.
        
        Run policy $\wh{\pi}_s$, and observe tuple $(s_n,a_{s,n},y_{s,n})$ for the final state, action, and feedback for this episode. 

        \If{$s_n=s$}{
            $D_s \leftarrow D_s \cup \{(x_{s,n}, s, a_{s,n}, y_{s,n})\}$, $n \leftarrow n+1$
        }
    }
}
\textbf{Output:} $D \triangleq \bigcup_{s\in \wh{\calS}_{H,\eps} } D_s$.
\end{algorithm}

\subsection{Omitted Proofs in \pref{sec: decoder}}\label{app: decoder_proof}

In this section, we show the omitted proofs in \pref{sec: decoder}. 
The first lemma shows that using \pref{alg:homing}, we are able to obtain a policy whose probability of reaching $s$ is close to the maximum achievable among all policies.

\stateIdentify*
\begin{proof}
    According to Lemma 3.4 of \citet{jin2020reward}, we know that by applying \textsc{Euler}~\citep{EULER2019}, with probability at least $1-\delta$, after $N=\frac{C\cdot S^2KH\log(SKH/\delta)}{\epsilon^2}$ rounds for some constant $C\geq 1$, for each state $s\in \calS_H$, we are able to obtain a policy such that the probability of visiting $s$ is no smaller than $p^\star_s-\epsilon$. Specifically, for each state, we obtain a set of policy $\Phi_s=\{\pi_{s,1},\dots, \pi_{s,N}\}$ such that when $\wh{\pi}_s$ is a uniform policy over $\Phi_s$, we have
    \begin{align*}
    P^{\wh{\pi}_s}(s)=\frac{1}{N}\sum_{n=1}^NP^{\pi_{s,n}}(s)\geq p_s^\star-\epsilon.
    \end{align*}
\end{proof}

The next lemma shows that we are able to reliably distinguish between states with high and low maximum reaching probability based
solely on the observable empirical visitation frequency $\wh{p}_s$ for each $s\in\calS_H$.

\identifySignificant*
\begin{proof}
Let $\calE_1$ be the event from~\pref{lem:state_identify} that for all $s\in\calS_H$,
$P^{\wh{\pi}_s}(s)\ge p_s^\star-\eps$, so $\Pr(\mathcal{E}_1)\ge 1-\delta$. Let $\mathcal{E}_2$ be the
the event that for all $s\in\calS_H$, $|\wh{p}_s-P^{\wh{\pi}_s}(s)|\le \beta$. Standard Hoeffding inequality shows that $\Pr(\mathcal{E}_2)\ge 1-\delta$.

Suppose that $\mathcal{E}_1$ and $\mathcal{E}_2$ hold. If $\wh{p}_s\geq \tau+\beta$, then $P^{\wh{\pi}_s}(s)\ge \tau$ holds directly based on $\mathcal{E}_2$ and by definition of $p_s^\star$, we know that $p_s^\star \geq P^{\wh{\pi}_s}(s)\geq \tau$. For (ii), if $\wh{p}_s\leq \tau-(\beta+\epsilon)$, according to \pref{lem:state_identify}, we know that
\begin{align*}
    p_s^\star \leq P^{\wh{\pi}_s}(s)+\eps \leq \wh{p}_s + \eps+\beta \leq \tau,
\end{align*}
where the first inequality is due to $\calE_1$ and the second inequality is due to $\calE_2$.
\end{proof}

The following lemma is an analog of Lemma 1 in \citet{zhang2024provably} in the context of MDP.
\posterior*
\begin{proof}
Direct calculation shows that
     \begin{align*}
    \Pr[a|x,y,s] &= \frac{\Pr[a|x,s]\cdot \Pr[y|x,a,s]}{\Pr[y|x,s]} \\
    &= \Pr[a|x,s]\frac{\Pr[r=1|x,a,s] \cdot\Pr[y|x,a,s,r=1] + \Pr[r=0|x,a,s] \cdot\Pr[y|x,a,s,r=0]}{\Pr[y|x,s]} \\
    &= \Pr[a|x,s]\frac{f^\star(x,s,a) \cdot\Pr[y|x,s,r=1] + (1-f^\star(x,s,a)) \cdot\Pr[y|x,s,r=0]}{\Pr[y|x,s]} \\
    &=\Pr[a|x,s]\frac{f^\star(x,s,a) \cdot\Pr[r=1|x,y,s]}{\Pr[r=1|x,s]} + \Pr[a|x,s]\frac{(1-f^\star(x,s,a)) \cdot\Pr[r=0|x,y,s]}{\Pr[r=0|x,s]} \\
    &=\Pr[a|x,s]\frac{f^\star(x,s,a) \cdot\Pr[r=1|x,y,s]}{\sum_{a'=1}^K\Pr[a'|x,s]\Pr[r=1|x,s,a']} + \Pr[a|x,s]\frac{(1-f^\star(x,s,a)) \cdot\Pr[r=0|x,y,s]}{\sum_{a'=1}^K\Pr[a'|x,s]\Pr[r=0|x,s,a']} \\
    &=\Pr[a|x,s]\frac{f^\star(x,s,a) \cdot\phi^\star(x,y,s)}{\sum_{a'=1}^K\Pr[a'|x,s]f^\star(x,s,a')} + \Pr[a|x,s]\frac{(1-f^\star(x,s,a)) \cdot(1-\phi^\star(x,y,s))}{\sum_{a'=1}^K\Pr[a'|x,s](1-f^\star(x,s,a'))},
\end{align*}
where the third equality uses the fact that $y$ is independent of the action given $x$, $s$, and $r$. Since $\Pr[a|x,s]=\frac{1}{K}$ for all $a\in[K]$ given any $x\in\calX$ and $s\in\calS_H$, we know that
\begin{align*}
    \Pr[a|x,y,s]=\frac{f^\star(x,s,a) \cdot\phi^\star(x,y,s)}{\sum_{a'=1}^Kf^\star(x,s,a')} + \frac{(1-f^\star(x,s,a)) (1-\phi^\star(x,y,s))}{K-\sum_{a'=1}^Kf^\star(x,s,a')}.
\end{align*}   
\end{proof}

The next lemma shows how we are able to interpret the true reward from the ground-truth posterior for both homogeneous and heterogeneous states.
\ikbinary*
\begin{proof}
Fix any context $x\in\calX$ and heterogeneous state $s\in\calS_H$. For conciseness, we denote
$f_a \triangleq f^\star(x,s,a)$, $Q \triangleq \sum_{a=1}^K f_a \le M$, and
$a^\star \in \arg\max_{a\in[K]} f_a$, so that $f_{a^\star}\ge \theta$.
We lower bound $\bigl\|h^\star(x,y,s)-\tfrac{1}{K}\mathbf{1}\bigr\|_\infty$ in two cases.

\paragraph{Case 1: $\phi^\star(x,y,s)=1$.}
By \pref{eqn:ik_h}, we know that $h_a^\star(x,y,s)=\frac{f_a}{Q}$ for all $a\in[K]$. Moreover, according to \pref{asm:identify}, we know that $K\theta -M >K\theta-\sigma M = \theta M>0$. Hence,
\begin{align*}
\left\|h^\star(x,y,s)-\tfrac{1}{K}\mathbf{1}\right\|_\infty
&\ge h_{a^\star}^\star(x,y,s)-\tfrac{1}{K}
= \frac{f_{a^\star}}{Q}-\frac{1}{K}
\ge \frac{\theta}{Q}-\frac{1}{K}
\ge \frac{\theta}{M}-\frac{1}{K}
= \frac{K\theta-M}{KM}.
\end{align*}
Since $M<\frac{K}{2}$, we have $K-M>M$, and therefore
$\frac{K\theta-M}{KM}\ge \frac{K\theta-M}{K(K-M)}$.

\paragraph{Case 2: $\phi^\star(x,y,s)=0$.}
By \pref{eqn:ik_h}, $h_a^\star(x,y,s)=\frac{1-f_a}{K-Q}$ for all $a\in[K]$. In particular,
\begin{align*}
\left\|h^\star(x,y,s)-\tfrac{1}{K}\mathbf{1}\right\|_\infty
&\ge \left|\tfrac{1}{K}-h_{a^\star}^\star(x,y,s)\right|
\geq \frac{1}{K}-\frac{1-f_{a^\star}}{K-Q}
\ge \frac{1}{K}-\frac{1-\theta}{K-Q}
= \frac{K\theta-Q}{K(K-Q)}.
\end{align*}
The function $u\mapsto \frac{K\theta-u}{K(K-u)}$ is decreasing for $u\in(0,K\theta)$ (since $\theta\le 1$),
and thus using $Q\le M<K\theta$ yields
\[
\frac{K\theta-Q}{K(K-Q)} \ge \frac{K\theta-M}{K(K-M)}.
\]

Combining the two cases, we obtain
\[
\left\|h^\star(x,y,s)-\tfrac{1}{K}\mathbf{1}\right\|_\infty
\ge \frac{K\theta-M}{K(K-M)},
\]
which proves the first statement for heterogeneous states.

The remaining two statements for heterogeneous states follow the proof of Lemma 2 in \citet{zhang2024provably}. Specifically, if $h_a^\star(x,y,s)\ge \frac{\theta}{M}$, then we must have $\phi^\star(x,y,s)=1$ because otherwise,
\[
h_a^\star(x,y,s)=\frac{1-f_a}{K-Q}\le \frac{1}{K-M}< \frac{\theta}{M},
\]
where the first inequality uses the fact that $Q\leq M$ and $f_a\geq 0$, and the last inequality is by \pref{asm:identify}. Moreover, for $a=\argmax_{a'\in[K]}f^\star(x,s,a)$, we know that
$f^\star(x,s,a)\ge\theta$. Therefore, when $\phi^\star(x,y,s)=1$,
\[
h_{a}^\star(x,y,s)=\frac{f^\star(x,s,a)}{\sum_{a'=1}^K f^\star(x,s,a')}
\ge \frac{\theta}{M}.
\]

Finally, if $s$ is a homogeneous state, then direct calculation shows that for all $a\in[K]$,
\[
h_a^\star(x,y,s)=\frac{1}{K}\bigl(\phi^\star(x,y,s)+1-\phi^\star(x,y,s)\bigr)=\frac{1}{K},
\]
and hence $\left\|h^\star(x,y,s)-\tfrac{1}{K}\mathbf{1}\right\|_\infty=0$.
\end{proof}

\decoderProperty*
\begin{proof}
    The proof follows a similar spirit as the proof of \pref{lem:decoder}. Fix any context $x\in\calX$, state $s\in\calS_H$, action $a\in[K]$, and feedback $y$.
Recall $\Delta(v)= \|v-\tfrac{1}{K}\mathbf{1}\|_\infty$ and $\xi=\frac{1}{2}(\frac{\theta}{M}-\frac{1}{K-M})$.

When $s$ is a heterogeneous state, by \pref{lem:decoder}, we know that $\Delta(h^\star(x,y,s))\ge \kappa$.
Hence, by \pref{eqn:reward_decoder},
\[
J(h^\star(x,y,s),a)= G\left(h_a^\star(x,y,s),\frac{\theta}{M}-\xi,\xi\right).
\]
We then consider two cases.
\paragraph{Case 1: $\phi^\star(x,y,s)=1$.}
Then $r(x,s,a)=\phi^\star(x,y,s)=1$. Since $G\left(h_a^\star(x,y,s),\frac{\theta}{M}-\xi,\xi\right)\in[0,1]$, we know that
\[
r(x,s,a)=1 \ge J(h^\star(x,y,s),a).
\]
Moreover, if $a=\argmax_{a'\in[K]}f^\star(x,s,a')$, then \pref{lem:decoder} shows that
$h_{a}^\star(x,y,s)\ge \tfrac{\theta}{M}$. This means that $$G\left(h_a^\star(x,y,s),\frac{\theta}{M}-\xi,\xi\right) = \frac{h_a^\star(x,y,s)-\frac{\theta}{M}+\xi}{\xi}\cdot\mathbbm{1}\left\{h_a^\star(x,y,s)\in\Big[\frac{\theta}{M}-\xi,\frac{\theta}{M}\Big)\right\}+\mathbbm{1}\left\{h_a^\star(x,y,s)\geq \frac{\theta}{M}\right\}=1,$$
and therefore $
J(h^\star(x,y,s),a)
=G\left(h_a^\star(x,y,s),\frac{\theta}{M}-\xi,\xi\right)
=1
=r(x,s,a).$

\paragraph{Case 2: $\phi^\star(x,y,s)=0$.}
Then $r(x,s,a)=0$. By \pref{eqn:ik_h}, we know that
$h_a^\star(x,y,s)=\frac{1-f^\star(x,s,a)}{K-\sum_{a'} f^\star(x,s,a')}\le \frac{1}{K-M}$.
Under \pref{asm:identify}, we have
$\frac{1}{K-M}\le \tfrac{\theta}{M}-\xi$, so $h_a^\star(x,y,s)\le \tfrac{\theta}{M}-\xi$ and hence
$G\left(h_a^\star(x,y,s),\frac{\theta}{M}-\xi,\xi\right)=0$. Thus
\[
J(h^\star(x,y,s),a)=G\left(h_a^\star(x,y,s),\frac{\theta}{M}-\xi,\xi\right)=0=r(x,s,a).
\]
Combining both cases proves the first statement for heterogeneous states. When $a=\argmax_{a'\in[K]}f^\star(x,s,a)$, we only need to show that when $\phi^\star(x,y,s)=1$, $J^\star(h^\star(x,y,s),a)=1$. Since $f^\star(x,s,a)\geq \theta$ according to \pref{asm:identify},  we know that
\begin{align*}
    h_a^\star(x,y,s) = \frac{f^\star(x,s,a)}{\sum_{a'=1}^Kf^\star(x,s,a')}\geq \frac{\theta}{M},
\end{align*}
meaning that 
\begin{align*}
J^\star(h^\star(x,y,s),a) = G\left(h_a^\star(x,y,s),\frac{\theta}{M}-\xi,\xi\right) = 1 = \phi^\star(x,y,s).
\end{align*}

When $s$ is a homogeneous state, then by \pref{lem:decoder}, we know that $h^\star(x,y,s)=\tfrac{1}{K}\mathbf{1}$. Hence,
$\Delta(h^\star(x,y,s))=0\le \kappa/2$. Therefore by \pref{eqn:reward_decoder}, we know that
$J(h^\star(x,y,s),a)=c$, which equals to $\E[r(x,s,a)]$ by the definition of homogeneous states.

To show that $J(v,a)$ is Lipschitz in $v$ with respect to $\ell_\infty$ norm, we fix any $a\in[K]$ and let $v,v'\in\Delta_K$. By the triangular inequality, we know that
\begin{align}\label{eq:Delta_lip}
|\Delta(v)-\Delta(v')|\le \|v-v'\|_\infty.
\end{align}

When $\Delta(v)$ and $\Delta(v')$ are both in $[\frac{\kappa}{2},\kappa)$, we can write
\[
J(v,a)=\mu_1(\Delta(v))\cdot c+\mu_2(\Delta(v))\cdot G\left(v_a,\frac{\theta}{M}-\xi,\xi\right),
\]
where $\mu_1(u)=\frac{2(\kappa-u)}{\kappa},\mu_2(u)=\frac{2u-\kappa}{\kappa}$. Then for any $u,u'\in[\kappa/2,\kappa]$,
\begin{align}\label{eqn:Lip_aux}
|\mu_1(u)-\mu_1(u')|=|\mu_2(u)-\mu_2(u')|\le \frac{2}{\kappa}|u-u'|.
\end{align}
Therefore, direct calculation shows that
\begin{align*}
    \left|J(v,a)-J(v',a)\right| &\leq c\cdot \left|\mu_1(\Delta(v))-\mu_1(\Delta(v'))\right| + G\left(v_a,\frac{\theta}{M}-\xi,\xi\right)\cdot \left|\mu_2(\Delta(v))-\mu_2(\Delta(v'))\right| \\
    &\qquad + \mu_2(\Delta(v'))\left|G\left(v_a,\frac{\theta}{M}-\xi,\xi\right) - G\left(v_a',\frac{\theta}{M}-\xi,\xi\right)\right| \\
    &\leq \frac{2c\|v-v'\|_\infty}{\kappa} + \frac{2\|v-v'\|_\infty}{\kappa} + \frac{\|v-v'\|_\infty}{\xi} \\
    &\leq \left(\frac{4}{\kappa}+\frac{1}{\xi}\right)\|v-v'\|_\infty, \tag{since $|c|\leq 1$}
\end{align*}
where the second inequality uses \pref{eq:Delta_lip}, \pref{eqn:Lip_aux}, and the fact that
\begin{equation}\label{eq:Gtilde_lip}
\left|G\left(v_a,\tfrac{\theta}{M}-\xi,\xi\right) - G\left(v_a',\tfrac{\theta}{M}-\xi,\xi\right)\right|
\le \frac{1}{\xi}|v_a-v'_a|
\le \frac{\|v-v'\|_\infty}{\xi}.
\end{equation}

Since on the other two regions, $J(v,a)$ is either constant when $\Delta(v)\le \kappa/2$ or equals $G(v_a,\frac{\theta}{M}-\xi,\xi)$ when $\Delta(v)\ge \kappa$, $J(v,a)$ is also $\frac{1}{\xi}$ in the other two regions. Finally, since $J$ is continuous at $\Delta=\kappa/2$ and $\Delta=\kappa$, we know that $J(v,a)$ is Lipschitz over the whole domain with Lipschitz constant $L=\max\{\frac{4}{\kappa}+\frac{1}{\xi},\frac{1}{\xi}\} = \frac{4}{\kappa}+\frac{1}{\xi}$.
\end{proof}
\section{Omitted Details in \pref{sec: regret}}
In this section, we provide omitted details in \pref{sec: regret}, including the algorithm for policy learning in \pref{app: policy_learning_alg}, the log-loss regression regret guarantees for empirical transition estimation in \pref{app:transition_estimation}, and the proof of our main theorem \pref{thm:final_regret} in \pref{app:regret_analysis}.

\subsection{Policy Learning Algorithm}\label{app: policy_learning_alg}
In this subsection, we provide the algorithm description of IGL policy learning for contextual MDPs. The pseudo code is shown in \pref{alg:online_mdp}.

\setcounter{AlgoLine}{0}
\begin{algorithm}[ht]
\caption{IGL Policy Learning for Contextual MDPs}
\label{alg:online_mdp}
   \nl \textbf{Input:} state set $\widehat{\mathcal{S}}_{H,\epsilon}$, decoder $\wh{h}$, function classes $\mathcal{F}$, $\gamma > 0$, online regression oracle $\AlgSq$.
   
   \nl \textbf{Initialize:} $N_1(s,a) \leftarrow 0$ and $N_1(s,a,s') \leftarrow 0$ for all $h \in [H-1]$, $(s,a) \in \Scal_h \times \Acal, s' \in \Scal_{h+1}$.
   
   \nl \For{$t=1, \dots, T$}{
       \nl Receive context $x_t$ and $\wh{f}_t$ from \AlgSq.

    \nl Estimate the transition as $\wh{P}_t(s' \mid s,a) = \frac{N_t(s,a,s') + 1}{N_t(s,a) + |\Scal_{h+1}|}$ for all $h \in [H-1], (s,a) \in \Scal_h \times \Acal, s' \in \Scal_{h+1}$. \label{line: transition_estimate}
       
       \nl Solve for occupancy measure $\wh{q}_t$: \label{line: opt}
       \begin{equation*}
       \begin{split}
           \wh{q}_t = \operatorname*{argmax}_{q \in \mu(\wh{P}_t)} \Big\{  \sum_{h, s,a} q_h(s,a) \wh{f}_t(x_t, s, a)  + \frac{1}{\gamma} \sum_{h,s,a} \log q_h(s,a) \Big\},
       \end{split}
       \end{equation*}
       where $\mu(P) \subseteq [0,1]^{H \times S \times A}$ is the set of all valid occupancy measures defined by the transition $P$.
       
       \nl For all $h \in [H]$ and $(s,a) \in \Scal_h \times \Acal$, set policy as:
       $$
       \pi_{t}(a \mid s)=\frac{\wh{q}_{t,h}(s,a)}{\sum_{a'=1}^K \wh{q}_{t,h}(s,a')}.
       $$
       
       \nl Execute $\pi_t$, get trajectory $\tau_t=(s_{t,1},a_{t,1},s_{t,2},\dots,s_{t,H},a_{t,H})$ and feedback $y_t$.

       \nl \If{$s_{t,H} \in \widehat{\mathcal{S}}_{H,\epsilon}$}{
           \nl $\tilde{r}_t = J(\wh{h}(x_t, y_t, s_{t,H}), a_{t,H})$. \label{line: eqn_proxy}

           \nl Update $\AlgSq$ with $(x_t, s_{t,H}, a_{t,H}, \tilde{r}_t)$. \label{line: update_oracle}
        }

       \nl \For{$h=1,2,\dots, H-1$}{
            \For{each $(s,a,s')\in \calS_h\times \calA\times \calS_{h+1}$}{
                \begin{align}
                    N_{t+1}(s,a) &= N_t(s,a)+\mathbbm{1}\{s=s_{t,h},a=a_{t,h}\},\label{eqn:sa}\\
                    N_{t+1}(s,a,s') &= N_t(s,a,s')+\mathbbm{1}\{s=s_{t,h},a=a_{t,h},s'=s_{t,h+1}\}.\label{eqn:sas}
                \end{align}
            }
       }
   }
\end{algorithm}

\subsection{Guarantee for Empirical Transition Estimation}
\label{app:transition_estimation}

In this subsection, we provide the performance guarantee for estimating the transition kernel using smoothed empirical mean shown in \pref{line: transition_estimate} of \pref{alg:online_mdp}. Specifically, we show that this smoothed empirical mean enjoys a $\otil(S^2K)$ log-loss regression regret guarantee, which satisfies the requirements of the analysis in \citet{levy2023efficient}.

\begin{theorem}[Log-loss guarantee for Laplace-smoothed transition estimator]
\label{thm:transition_logloss}
Let $\{\widehat P_t\}_{t=1}^T$ be the sequence of Laplace-smoothed estimators defined in \pref{line: transition_estimate} and recall that $\calP$ is the true transition kernel.
Then, the cumulative log-loss regression regret with respect to $\calP$ is bounded by:
\begin{equation*}
\sum_{t=1}^T \sum_{h=1}^{H-1} \left( -\log \widehat P_{t}(s_{t,h+1} \mid s_{t,h}, a_{t,h}) \right)
-
\sum_{t=1}^T \sum_{h=1}^{H-1} \left( -\log \calP(s_{t,h+1} \mid s_{t,h}, a_{t,h}) \right)
\le
S^2K\log(TH+S),
\end{equation*}
\end{theorem}

Before proving \pref{thm:transition_logloss}, we first provide two helpful auxiliary lemmas.

\begin{lemma}
\label{lem:dirichlet_identity}
Fix a layer $h\in[H-1]$ and state-action pair $(s,a)\in \calS_h\times\calA$. Consider the subsequence of global time indices $t_1, \dots, t_B\in [T]$ such that $(s_{t_b, h}, a_{t_b, h}) = (s,a)$ for all $b\in[B]$.
Let $z_b \triangleq s_{t_b,h+1} \in \mathcal{S}_{h+1}$ be the next state observed at time $t_b$.
Let $n_{s'}$ denote the total number of times the outcome $s'\in\calS_{h+1}$ occurs in $\{(s_{t,1},a_{t,1},\dots,s_{t,H},a_{t,H})\}_{t\in[T]}$.
Then the product of the sequential predictive probabilities satisfies:
\begin{align*}
    \prod_{b=1}^B \wh{P}_{t_b}(z_b \mid s,a)
=
\frac{(S_{h+1}-1)!\,\prod_{s'\in\mathcal{S}_{h+1}} n_{s'}!}{(S_{h+1}+B-1)!},
\end{align*}
where we denote $|\calS_h|=S_h$ for all $h\in[H]$.
\end{lemma}

\begin{proof}
Recall the estimator definition: $\wh{P}_{t_b}(z_b \mid s,a) = \frac{N_{t_b}(s,a,z_b) + 1}{N_{t_b}(s,a) + S_{h+1}}$.
We analyze the numerators and denominators of $\prod_{b=1}^B \wh{P}_{t_b}(z_b \mid s,a)$ separately.

\textbf{Denominator.} Note that $N_{t_b}(s,a)$ counts the visitations to $(s,a)$ prior to time $t_b$. Thus, as $b$ goes from $1$ to $B$, $N_{t_b}(s,a)$ takes values $0, 1, \dots, B-1$. The product of denominators is:
\[
\prod_{b=1}^B (N_{t_b}(s,a) + S_{h+1}) = \prod_{k=0}^{B-1} (S_{h+1} + k) = \frac{(S_{h+1}+B-1)!}{(S_{h+1}-1)!}.
\]

\textbf{Numerator.} We regroup the product terms based on the outcome state $s'$. For a fixed state $s'\in\calS_{h+1}$, the term $(N_{t_b}(s,a,s') + 1)$ appears exactly $n_{s'}$ times. Across these occurrences, the count $N_{t_b}(s,a,s')$ increments sequentially from $0$ to $n_{s'}-1$. Thus, the contribution from state $s'$ is:
\[
\prod_{k=0}^{n_{s'}-1} (k + 1) = n_{s'}!.
\]
Taking the product over all $s' \in \mathcal{S}_{h+1}$ yields $\prod_{s'\in\calS_{h+1}} n_{s'}!$.

Combining the calculation of the numerator and the denominator concludes the proof.
\end{proof}

\begin{lemma}
\label{lem:entropy_bound}
Fix any $h \in [H-1]$ and let $S_{h+1} \coloneqq |\mathcal{S}_{h+1}|$. Let $\{n_{s'}\}_{s' \in \mathcal{S}_{h+1}}$ be a set of non-negative integers such that $\sum_{s'} n_{s'} = B$. Define the distribution $q \in \Delta_{\mathcal{S}_{h+1}}$ such that $q(s') \triangleq n_{s'}/B$ for all $s'\in\calS_{h+1}$. Then, the following inequality holds:
\[
    \log \frac{(S_{h+1} + B - 1)!}{(S_{h+1} - 1)! \prod_{s'} n_{s'}!}
    \;\le\;
    -B \sum_{s' \in \mathcal{S}_{h+1}} q(s') \log q(s') + (S_{h+1} - 1)\log(S_{h+1} + B).
\]
\end{lemma}

\begin{proof}
We decompose the term inside the logarithm into the product of a multinomial coefficient and a binomial coefficient:
\[
\frac{(S_{h+1}+B-1)!}{(S_{h+1}-1)!\,\prod_{s'} n_{s'}!}
=
\underbrace{\frac{B!}{\prod_{s'} n_{s'}!}}_{\textsc{Term-A}}
\cdot
\underbrace{\binom{S_{h+1}+B-1}{S_{h+1}-1}}_{\textsc{Term-B}}.
\]

\textbf{Bounding \textsc{Term-A}.} Consider the multinomial expansion of $(\sum_{s'} x_{s'})^B$ with $x_{s'} = q(s') = n_{s'}/B$. Since $\sum_{s'} q(s') = 1$, we have:
\[
1 = \left(\sum_{s'} q(s')\right)^B = \sum_{k_1+\dots+k_{S_{h+1}}=B} \frac{B!}{\prod_{s'} k_{s'}!} \prod_{s'} q(s')^{k_{s'}}.
\]
This sum is lower bounded by the single term where the exponents match the counts (i.e., $k_{s'} = n_{s'}$):
\[
1 \ge \frac{B!}{\prod_{s'} n_{s'}!} \prod_{s'} q(s')^{n_{s'}}.
\]
Taking the logarithm and rearranging terms yields:
\[
\log (\textsc{Term-A}) \le - \sum_{s'} n_{s'} \log q(s') = B \sum_{s'} -\frac{n_{s'}}{B} \log q(s') = -B\sum_{s'}q(s')\log q(s').
\]

\textbf{Bounding \textsc{Term-B}.} Using the inequality $\binom{n}{k} \le (n+1)^k$:
\[
\binom{S_{h+1}+B-1}{S_{h+1}-1} < (S_{h+1}+B)^{S_{h+1}-1}.
\]
Taking the logarithm gives $\log(\textsc{Term-B}) < (S_{h+1}-1)\log(S_{h+1}+B)$. Summing the bounds for \textsc{Term-A} and \textsc{Term-B} concludes the proof.
\end{proof}

With these technical lemmas established, we are now ready to present the main theorem regarding the log-loss regret of our estimated transition model.

\begin{proof}[Proof of \protect\pref{thm:transition_logloss}]
We partition the transitions according to the state-action pairs. Fix a layer $h \in [H-1]$ and a state-action pair $(s,a) \in \mathcal{S}_h \times \mathcal{A}$.
Let $t_1, \dots, t_{m_{s,a}}$ be the subsequence of episode indices where $(s_{t_j, h}, a_{t_j, h}) = (s,a)$, and let $z_j \triangleq s_{t_j, h+1} \in \mathcal{S}_{h+1}$ be the observed next states. Let $S_{h+1}=|\calS_{h+1}|$ and $q(s') \coloneqq n_{s'}/m_{s,a}$ be the empirical distribution of outcomes. Then, following \pref{lem:dirichlet_identity} and \pref{lem:entropy_bound}, we can bound $\sum_{j=1}^{m_{s,a}} \left( -\log \widehat P_{t_j}(z_j \mid s,a) \right) - \sum_{j=1}^{m_{s,a}} \left( -\log \calP(z_j \mid s,a) \right)$ as follows:
\begin{align*}
    & \sum_{j=1}^{m_{s,a}} \left( -\log \widehat P_{t_j}(z_j \mid s,a) \right) - \sum_{j=1}^{m_{s,a}} \left( -\log \calP(z_j \mid s,a) \right) \\
    &= -\log \left( \prod_{j=1}^{m_{s,a}} \widehat P_{t_j}(z_j \mid s,a) \right) + \sum_{s' \in \mathcal{S}_{h+1}} n_{s'} \log \calP(s' \mid s,a) 
     \\
    &= \log \left( \frac{(S_{h+1}+m_{s,a}-1)!}{(S_{h+1}-1)!\,\prod_{s'} n_{s'}!} \right) + m_{s,a} \sum_{s'} q(s') \log \calP(s' \mid s,a) 
    \tag{by \pref{lem:dirichlet_identity} and  $n_{s'} = m_{s,a} q(s') $} \\
    &\le \left[ -m_{s,a} \sum_{s'} q(s') \log q(s') + C_{s,a} \right] + m_{s,a} \sum_{s'} q(s') \log \calP(s'|s,a) 
    \tag{By \pref{lem:entropy_bound}} \\
    &= -m_{s,a} \sum_{s'} q(s') \left( \log q(s') - \log \calP(s'|s,a) \right) + C_{s,a} 
    \tag{rearranging terms} \\
    &\le (S_{h+1}-1)\log(S_{h+1} + m_{s,a}),
    \tag{since KL-distance is non-negative}
\end{align*}
where we define $C_{s,a}\triangleq (S_{h+1}-1)\log(S_{h+1}+m_{s,a})$.

Taking summation over all state-action pairs and all layers, we know that
\begin{align*}
    &\sum_{t=1}^T \sum_{h=1}^{H-1} \left( -\log \widehat P_{t}(s_{t,h+1} \mid s_{t,h}, a_{t,h}) \right)
-
\sum_{t=1}^T \sum_{h=1}^{H-1} \left( -\log \calP(s_{t,h+1} \mid s_{t,h}, a_{t,h}) \right) \\
&=\sum_{h=1}^{H-1}\sum_{(s,a)\in\calS_h\times\calA}\sum_{j=1}^{m_{s,a}} \left( -\log \widehat P_{t_j}(z_j \mid s,a) \right) - \sum_{h=1}^{H-1}\sum_{(s,a)\in\calS_h\times\calA}\sum_{j=1}^{m_{s,a}} \left( -\log \calP(z_j \mid s,a) \right) \\
&\leq \sum_{h=1}^{H-1}S_{h+1}^2K\log(S+TH) \tag{since $m_{s,a}\leq TH$ and $S_{h+1}\leq S$}\\
&\leq \left(\sum_{h=1}^{H-1}S_{h+1}\right)^2K\log(S+TH) \leq S^2K\log(S+TH).
\end{align*}
\end{proof}

\subsection{Proof of \pref{thm:final_regret}}
\label{app:regret_analysis}
In this subsection, we prove our main result, \pref{thm:final_regret}. We set the confidence parameter $\delta = 1/T^2$ for the high-probability bounds in \pref{sec: decoder} (Step 1 and Step 2). Consequently, the expected regret from failure events is bounded by $T \cdot \delta = \mathcal{O}(1/T)$. In the remainder of this analysis, the expectation is taken on the condition that these high-probability events hold.

We decompose the total expected regret into three components: the regret incurred during the exploration phases (\pref{alg:homing} and \pref{alg:tuple_collection}), the regret accumulated during the online policy learning phase (\pref{alg:online_mdp}), and the regret attributable to failure events. Letting $T_{\text{start}}$ denote the first episode of the online phase, we write:
\begin{equation}
    \Reg = \underbrace{\E\left[\sum_{t=1}^{T_{\text{start}}-1} (V^\star - V(\pi_t))\right]}_{\Reg_{\textsc{explore}}} + \underbrace{\E\left[\sum_{t=T_{\text{start}}}^{T} (V^\star - V(\pi_t))\right]}_{\Reg_{\textsc{online}}} + \mathcal{O}(1), \nonumber
\end{equation}
where $\mathcal{O}(1)$ bounds the regret incurred under failure events. In the following subsections, we bound these two terms separately.

\subsubsection{Bounding $\Reg_{\textsc{explore}}$}
The exploration phase $[T_{\text{start}}-1]$ consists of two parts: learning the homing policy for each state $s\in\calS_H$ using \pref{alg:homing} and collecting tuples for reachable states $s\in\wh{\calS}_{H,\epsilon}$ using \pref{alg:tuple_collection}.

For the first part, according to \pref{lem:state_identify}, under the high-probability events, learning homing policies for all $s \in \mathcal{S}_H$ using \pref{alg:homing} requires $\otil(\frac{S^2 K H}{\epsilon^2})$ episodes, leading to an $\otil(\frac{S^2 K H}{\epsilon^2})$ regret.

For the second part, according to \pref{lem:sampling}, under the high-probability events, collecting $N_0$ samples for all reachable states $s\in \widehat{\mathcal{S}}_{H,\epsilon}$ using \pref{alg:tuple_collection} requires $ \otil(\frac{S N_0}{\epsilon})$ episodes, leading to an $\otil(\frac{S N_0}{\epsilon})$ regret.
Thus, $\Reg_{\text{explore}}$ is bounded as:
\begin{equation}
\label{eq:bound_explore}
    \Reg_{\textsc{explore}} \le \otil\left( \frac{S^2 K H}{\epsilon^2} + \frac{S N_0}{\epsilon} \right).
\end{equation}

\subsubsection{Bounding $\Reg_{\textsc{online}}$}

We analyze the regret by transferring the problem to a surrogate MDP restricted to the reachable states. We recall the definition of the surrogate function $\underline{f}^\star$ as:
\[
\underline{f}^\star(x, s, a) \triangleq  \E_{y \mid x,s,a} [J(h^\star(x,y,s), a)].
\]
We define the value of a policy $\pi$ under an expected reward function $f$ as $V(\pi; f) \triangleq \E_{\pi}[ f(x, s_{H}^{\pi}, a_{H}^{\pi})]$ where $(s_H^{\pi},a_H^{\pi})$ is sampled from the marginal distribution at step $H$ induced by $\pi$. Then, we bound the online regret as follows:
\begin{align}
    \Reg_{\textsc{online}}
    &= \E\left[\sum_{t=T_{\text{start}}}^{T} \left( V(\pi^\star; f^\star) - V(\pi_t; f^\star) \right)\right] \nonumber \\
    &= \E\left[\sum_{t=T_{\text{start}}}^{T} \left( {f}^\star(x_t,s_{t,H}^{\pi^\star},a_{t,H}^{\pi^\star}) - {f}^\star(x_t,s_{t,H},a_{t,H})\right) \mathbbm{1}\{s_{t,H}\in\wh{\calS}_{H,\epsilon}\}\right] \nonumber \\
    &\qquad +\E\left[\sum_{t=T_{\text{start}}}^{T} \left( {f}^\star(x_t,s_{t,H}^{\pi^\star},a_{t,H}^{\pi^\star}) - {f}^\star(x_t,s_{t,H},a_{t,H})\right) \mathbbm{1}\{s_{t,H}\notin\wh{\calS}_{H,\epsilon}\}\right]\nonumber\\
    &\leq \E\left[\sum_{t=T_{\text{start}}}^{T} \left( {f}^\star(x_t,s_{t,H}^{\pi^\star},a_{t,H}^{\pi^\star}) - {f}^\star(x_t,s_{t,H},a_{t,H})\right) \mathbbm{1}\{s_{t,H}\in\wh{\calS}_{H,\epsilon}\}\right] + \order(TS\epsilon)  \nonumber\\
    &\leq \E\left[\sum_{t=T_{\text{start}}}^{T} \left( \underline{f}^\star(x_t,s_{t,H}^{\pi^\star},a_{t,H}^{\pi^\star}) - \underline{f}^\star(x_t,s_{t,H},a_{t,H})\right) \mathbbm{1}\{s_{t,H}\in\wh{\calS}_{H,\epsilon}\}\right] + \order(TS\epsilon) ,\label{eq:online_1}
\end{align}
where the first inequality holds since $\calP(s)\leq 6\epsilon$ for all $s\in\wh{\calS}_{H,\epsilon}$ according to \pref{lem:identify_significant_state} and the second inequality is due to \pref{lem:ambigurous_state_identify} as for any $x\in\calX$, $s\in\calS_H$, and $a\in[K]$,
\begin{align*}
{f}^\star(x,s,a) = \E_r[r(x,s,a)] \geq \E_y[J(h^\star(x,y,s),a)] =  \underline{f}^\star(x,s,a),
\end{align*}
and when $a=\argmax_{a'\in[K]}f^\star(x,s,a)$,
\begin{align*}
    {f}^\star(x,s,a) = \E_r[r(x,s,a)] = \E_y[J(h^\star(x,y,s),a)] =  \underline{f}^\star(x,s,a).
\end{align*}

We define the cumulative expected estimation errors for the reward and dynamics as follows:
\begin{align*}
    \mathcal{E}_{\text{rew}} &\triangleq  \mathbb{E} \left[ \sum_{t=T_{\text{start}}}^{T}\left(\wh{f}_t(x_t, s_{t,H}, a_{t,H}) - \underline{f}^\star(x_t, s_{t,H}, a_{t,H})\right)^2\cdot \mathbbm{1}\{s_{t,H}\in\wh{\calS}_{H,\epsilon}\} \right], \\
    \mathcal{E}_{\text{dyn}} &\triangleq \mathbb{E}\left[\sum_{t=T_{\text{start}}}^T \sum_{h=1}^{H-1} \left( -\log \widehat P_{t}(s_{t,h+1} \mid s_{t,h}, a_{t,h}) + \log \calP(s_{t,h+1} \mid s_{t,h}, a_{t,h}) \right) \right].
\end{align*}
Following the proof of Theorem 5 in \citet{levy2023efficient} and Lemma 9 in \citet{levy2023efficient}, the regret on the surrogate MDP is bounded by:
\begin{align*}
     &\E\left[\sum_{t=T_{\text{start}}}^{T} \left( \underline{f}^\star(x_t,s_{t,H}^{\pi^\star},a_{t,H}^{\pi^\star}) - \underline{f}^\star(x_t,s_{t,H},a_{t,H})\right) \mathbbm{1}\{s_{t,H}\in\wh{\calS}_{H,\epsilon}\}\right]\\
    &\le \tilde{O}\left(\frac{HSKT}{\gamma}+\gamma \mathcal{E}_{\text{rew}} + \gamma H^4(\mathcal{E}_{\text{dyn}}+H)+\sqrt{TH(\mathcal{E}_{\text{dyn}}+H)}+\sqrt{TH \mathcal{E}_{\text{rew}}}\right).
\end{align*}
Substituting the bound $\mathcal{E}_{\text{dyn}} \le \otil (S^2 K)$ according to~\pref{thm:transition_logloss}, we further obtain that:
\begin{align}
    &\E\left[\sum_{t=T_{\text{start}}}^{T} \left( \underline{f}^\star(x_t,s_{t,H}^{\pi^\star},a_{t,H}^{\pi^\star}) - \underline{f}^\star(x_t,s_{t,H},a_{t,H})\right) \mathbbm{1}\{s_{t,H}\in\wh{\calS}_{H,\epsilon}\}\right] \nonumber\\
    &\le \tilde{O}\left(\frac{HSKT}{\gamma}+\gamma \mathcal{E}_{\text{rew}} + \gamma H^4 S^2 K + \sqrt{TH S^2 K} + \sqrt{TH \mathcal{E}_{\text{rew}}}\right).\label{eqn:gamma_bound}
\end{align}

Next, we bound the reward estimation error $\calE_{\text{rew}}$. Let $\calT_{\text{reach}} \triangleq \{t \in \{T_{\text{start}},T_{\text{start}}+1,\dots,T\} : s_{t,H} \in \wh{\calS}_{H,\epsilon}\}$ denote the set of time indices where the reachable states are visited. For brevity, let $\zeta_t \triangleq (x_t, s_{t,H}, a_{t,H})$. We also define the expected proxy reward as $\bar{f}(\zeta_t) \triangleq \E_{y_t}[J(\wh{h}(x_t,y_t,s_{t,H}))] \cdot \mathbbm{1}\{s_{t,H} \in \wh{\calS}_{H,\epsilon}\}$.

Invoking the regression oracle guarantee (Assumptions~\ref{asm:reward_oracle} and~\ref{asm:realizability_lower}) over the subsequence $\calT_{\text{reach}}$, we have:
\begin{equation}
\label{eq:oracle_application}
    \sum_{t\in\calT_{\text{reach}}} (\wh{f}_t(\zeta_t) - \tilde{r}_t)^2 - \sum_{t\in\calT_{\text{reach}}} (\underline{f}^\star(\zeta_t) - \tilde{r}_t)^2 \le \RegSq.
\end{equation}
Using the identity $(a-b)^2 - (c-b)^2 = (a-c)^2 + 2(a-c)(c-b)$ with $a=\wh{f}_t(\zeta_t)$, $b=\tilde{r}_t$, and $c=\underline{f}^\star(\zeta_t)$, we can rewrite \pref{eq:oracle_application} as follows:
\begin{align}\label{eq:oracle_application_1}
\sum_{t\in\calT_{\text{reach}}} (\wh{f}_t(\zeta_t) - \underline{f}^\star(\zeta_t))^2 + 2 \sum_{t\in\calT_{\text{reach}}} (\wh{f}_t(\zeta_t) - \underline{f}^\star(\zeta_t))(\underline{f}^\star(\zeta_t) - \tilde{r}_t) \le \RegSq.
\end{align}
We now take the conditional expectation with respect to the feedback $y_t$ given $\zeta_t$. Since $\wh{f}_t$ is determined by the history up to $\zeta_t$ and $\underline{f}^\star$ is independent of the history up to $\zeta_t$, and noting that $\E_{y_t}[\tilde{r}_t \mid \zeta_t] = \bar{f}(\zeta_t)$, the expected cross-term satisfies that:
\begin{align}
    \E_{y_t} \left[ (\wh{f}_t(\zeta_t) - \underline{f}^\star(\zeta_t))(\underline{f}^\star(\zeta_t) - \tilde{r}_t) \right]
    &= (\wh{f}_t(\zeta_t) - \underline{f}^\star(\zeta_t))(\underline{f}^\star(\zeta_t) - \E_{y_t}[\tilde{r}_t]) \nonumber \\
    &= (\wh{f}_t(\zeta_t) - \underline{f}^\star(\zeta_t))(\underline{f}^\star(\zeta_t) - \bar{f}(\zeta_t)). \label{eq:cross_term_exp}
\end{align}
Finally, we lower bound this term using the AM-GM inequality variation $ab \ge -\frac{1}{4}a^2 - b^2$ (by setting $a = \wh{f}_t(\zeta_t) - \underline{f}^\star(\zeta_t)$ and $y = \underline{f}^\star(\zeta_t) - \bar{f}(\zeta_t)$):
\begin{align}\label{eq:oracle_application_2}
(\wh{f}_t(\zeta_t) - \underline{f}^\star(\zeta_t))\cdot (\underline{f}^\star(\zeta_t) - \bar{f}(\zeta_t))
\ge
-\frac{1}{4}(\wh{f}_t(\zeta_t) - \underline{f}^\star(\zeta_t))^2 - (\underline{f}^\star(\zeta_t) - \bar{f}(\zeta_t))^2.
\end{align}
Taking the total expectation and rearranging the terms, we can obtain that:
\begin{align*}
    \mathcal{E}_{\text{rew}} &=\E\left[\sum_{t\in\calT_{\text{reach}}}(\wh{f}_t(\zeta_t)-\underline{f}^\star(\zeta_t))^2\right] \\
    &\leq \RegSq - 2\E\left[\sum_{t\in\calT_{\text{reach}}} (\wh{f}_t(\zeta_t) - \underline{f}^\star(\zeta_t))(\underline{f}^\star(\zeta_t) - \tilde{r}_t)\right] \tag{according to \pref{eq:oracle_application_1}}\\
    &= \RegSq + \frac{1}{2}\E\left[\sum_{t\in\calT_{\text{reach}}} (\wh{f}_t(\zeta_t) - \underline{f}^\star(\zeta_t))^2 \right] + 2\E\left[\sum_{t\in\calT_{\text{reach}}} (\underline{f}^\star(\zeta_t) - \bar{f}(\zeta_t))^2 \right]\tag{according to \pref{eq:oracle_application_2}} \\
    &= \RegSq + \frac{1}{2}\calE_{\text{rew}} + 2\E\left[\sum_{t\in\calT_{\text{reach}}} (\underline{f}^\star(\zeta_t) - \bar{f}(\zeta_t))^2 \right].
\end{align*}
Rearranging the terms leads to the following:
\begin{equation}\label{eq:rearranging}
    \mathcal{E}_{\text{rew}} \le 2 \RegSq + 4  \E \left[\sum_{t\in\calT_{\text{reach}}} (\underline{f}^\star(\zeta_t) - \bar{f}(\zeta_t))^2 \right].
\end{equation}

We now bound the second term on the RHS of \pref{eq:rearranging}. Using the definitions of $\underline{f}^\star$ and $\bar{f}$, we apply Jensen's inequality and the $L$-Lipschitz property of $J(\cdot, a_{t,H})$ with respect to the $\ell_\infty$-norm:
\begin{align}
    (\underline{f}^\star(\zeta_t) - \bar{f}(\zeta_t))^2 \cdot \mathbbm{1}\{s_{t,H} \in \widehat{\mathcal{S}}_{H,\epsilon}\}
    &= \mathbbm{1}\{s_{t,H} \in \widehat{\mathcal{S}}_{H,\epsilon}\} \cdot \left( \E_{y_t}[J(h^\star(\zeta_t'), y_t) - J(\wh{h}(\zeta_t'), y_t)] \right)^2 \nonumber \\
    &\le \mathbbm{1}\{s_{t,H} \in \widehat{\mathcal{S}}_{H,\epsilon}\} \cdot \E_{y_t} \left[ (J(h^\star(\zeta_t'), y_t) - J(\wh{h}(\zeta_t'), y_t))^2 \right] \nonumber \\
    &\le L^2 \E_{y_t} \left[ \|h^\star(\zeta_t') - \wh{h}(\zeta_t')\|_\infty^2 \right] \cdot \mathbbm{1}\{s_{t,H} \in \widehat{\mathcal{S}}_{H,\epsilon}\}, \nonumber
\end{align}
where we used the shorthand $\zeta_t' \triangleq (x_t, y_t, s_{t,H})$. 

To invoke \pref{lem:erm}, we relate the online policy $\pi_t$ to the uniform exploration policy $q_{\text{unif}}$. As the density ratio is bounded by $K$, we know that
\begin{align}
    \E \left[ \sum_{t\in\calT_{\text{reach}}} (\underline{f}^\star(\zeta_t) - \bar{f}(\zeta_t))^2 \right]
    &\le L^2 \E_{x_t \sim \mathcal{D}, \pi_t} \left[\sum_{t\in\calT_{\text{reach}}} \E_{y_t} [\|h^\star(\zeta_t') - \wh{h}(\zeta_t')\|_\infty^2] \right] \nonumber \\
    &\le L^2  K \cdot \E_{\substack{x \sim \mathcal{D} \\ a \sim q_{\text{unif}}}} \left[ \sum_{t\in\calT_{\text{reach}}}\E_{y} \left[\|h^\star(x, y, s_{t,H}) - \wh{h}(x, y, s_{t,H})\|_2^2\right] \right]. \nonumber
\end{align}
By \pref{lem:erm}, we know that 
\begin{align*}
    \E_{\substack{x \sim \mathcal{D} \\ a \sim q_{\text{unif}}}} \left[ \sum_{t\in\calT_{\text{reach}}}\E_{y} \left[\|h^\star(x, y, s_{t,H}) - \wh{h}(x, y, s_{t,H})\|_2^2\right] \right]\leq \otil\left(\frac{T}{N_0}\right).
\end{align*}
Plugging the above back to \pref{eq:rearranging}, we conclude that
\begin{equation}
\mathcal{E}_{\text{rew}} \le 2 \RegSq + \otil\left( \frac{T K L^2}{N_0} \right). \nonumber
\end{equation}

Substituting the bound for $\mathcal{E}_{\text{rew}}$ back into \pref{eqn:gamma_bound}, we obtain that

\begin{align}
    &\E\left[\sum_{t=T_{\text{start}}}^{T} \left( \underline{f}^\star(x_t,s_{t,H}^{\pi^\star},a_{t,H}^{\pi^\star}) - \underline{f}^\star(x_t,s_{t,H},a_{t,H})\right) \mathbbm{1}\{s_{t,H}\in\wh{\calS}_{H,\epsilon}\}\right] \nonumber\\
    &\le \tilde{O}\left(\frac{HSKT}{\gamma}+\gamma \mathcal{E}_{\text{rew}} + \gamma H^4 S^2 K + \sqrt{TH S^2 K} + \sqrt{TH \mathcal{E}_{\text{rew}}}\right)\nonumber\\
    &\le \tilde{O}\left(\frac{HSKT}{\gamma}+\gamma\RegSq +
 \frac{\gamma TKL^2}{N_0}+ \gamma H^4 S^2 K + \sqrt{TH S^2 K} + \sqrt{TH \RegSq} + TL\sqrt{\frac{KH}{N_0}}\right).\label{eq:online_2}
\end{align}

\paragraph{Final Bound.}
Combining the exploration regret bound (\pref{eq:bound_explore}) with the online regret bounds (\pref{eq:online_1}) and (\pref{eq:online_2}), the total regret is bounded by:
\begin{align*}
    \Reg &\le  \otil\left( \frac{S^2 K H}{\epsilon^2} + \frac{S N_0}{\epsilon} \right) + \otil(T S \epsilon) \\
    &\qquad +\tilde{O}\left(\frac{HSKT}{\gamma}+\gamma\RegSq +
 \frac{\gamma TKL^2}{N_0}+ \gamma H^4 S^2 K + \sqrt{TH S^2 K} + \sqrt{TH \RegSq} + TL\sqrt{\frac{KH}{N_0}}\right)\\
 &=\otil\left(T^{\frac{3}{4}}S^{\frac{1}{2}}K^{\frac{1}{4}}H^{\frac{1}{4}}L^{\frac{1}{2}}+T^{\frac{2}{3}}S^{\frac{4}{3}}K^{\frac{1}{3}}H^{\frac{1}{3}}+\sqrt{TS^3K^2H^5}+\sqrt{TSKH\RegSq}\right),
    \nonumber
\end{align*}
where the last equality is by picking $\gamma$, $N_0$, and $\epsilon$ optimally. Specifically, we pick these parameters as follows: $\gamma=\min\left\{\sqrt{\frac{TSKH}{\RegSq}}, \sqrt{\frac{T}{SKH^3}},T^{\frac{1}{4}}S^{\frac{1}{4}}K^{\frac{1}{2}}L^{-\frac{1}{2}}H^{\frac{3}{4}}\right\}$, $N_0 =\max\left\{\frac{L\sqrt{TKH}}{S},\gamma^{\frac{2}{3}}T^{\frac{1}{3}}S^{-\frac{2}{3}}K^{\frac{2}{3}}L^{\frac{4}{3}}\right\}$, $\eps = \max\left\{\sqrt{\frac{N_0}{T}},T^{-\frac{1}{3}}S^{\frac{1}{3}}K^{\frac{1}{3}}H^{\frac{1}{3}}\right\}$.

\section{Omitted Details in \pref{sec: experiment}}\label{app: exp}

\subsection{Implementation Details in \pref{sec: synthe}}\label{app: synthe}
In this subsection, we present the omitted implementation details for experiments in \pref{sec: synthe}.
Our implementation follows the two-stage pipeline detailed in \pref{sec: decoder} and \pref{sec: regret}.
In the {decoder-learning phase}, we first run \pref{alg:homing} for each candidate terminal state $s \in \calS_3$ with $\delta=0.05$ and $N=5000$ to obtain the homing policies $\wh{\pi}_s$.
We then execute \pref{alg:tuple_collection} with  $N_0=5000$ to collect a dataset of context--state--action--signal tuples.
To approximate the posterior distribution $\wh{h}$, we implement the function class $\calH$ by parameterizing both the decoder class $\Phi$ and the reward function class $\calF$ as two-layer fully-connected neural networks.
We then solve the ERM problem by training a separate classifier for each terminal state using the collected tuples.
Finally, we convert the learned probabilities into a reward estimator using the Lipschitz construction in \pref{eqn:reward_decoder}.

In the {policy-learning phase}, we apply \pref{alg:online_mdp} with a learning rate schedule $\gamma_t = H\sqrt{|\calA|t}$ over a total horizon of $T=40,000$ episodes. The regression oracle is implemented by online gradient descent with learning rate $0.05$.
Crucially, the learner updates its policy solely based on the decoded rewards; ground-truth rewards are used exclusively for performance evaluation and plotting.

\subsection{Dataset Generation Details in~\pref{sec: real}}\label{app: real_data}
In this subsection, we present the omitted data generation details in~\pref{sec: real}. Each episode involves a two-turn dialogue between the agent and the user, corresponding to a \textbf{horizon} $H=2$. In the first turn, the agent responds to a user request by asking a clarification question. The ground-truth question (taken from the dataset) is the most informative, containing all necessary details to complete the booking (e.g., cuisine type and party size for restaurants; genre and location for movies). Conversely, we synthesize two incorrect candidate questions using an LLM. While contextually relevant, these questions fail to include sufficient information to fulfill the request.

In the second turn, the agent finalizes the booking based on the information gathered. Alongside the correct booking action from the dataset, we used an LLM to generate four incorrect booking options, each containing details that deviate from the ground truth. After the agent provides the booking information, we use an LLM to simulate user feedback, indicating whether the user is satisfied or unsatisfied based on the correctness and sufficiency of the information. We provide the specific prompts used for data generation as follows.

\paragraph{Clarification Questions.}
We present the prompt used to generate incorrect clarification questions, which correspond to the \textbf{intermediate actions} $a_h$.
The ``\texttt{initial\_request}'' represents the input from the booking dataset, serving as the \textbf{context} $x$ in our formulation.

\begin{promptbox}
You are helping design clarification questions for a movie ticket booking assistant. User's initial request: \{initial\_request\}. Generate one question in English that is about movie booking and sounds natural, but does NOT help the assistant gather the key information needed to complete the booking correctly. For example, ask about irrelevant details like movie reviews, snack preferences, parking, or general theater facilities instead of essential booking details (movie, theater, time, date, tickets). Keep it short and specific to the movie booking context. Return ONLY the question text.
\end{promptbox}
\begin{promptbox}
You are helping design clarification questions for a restaurant reservation assistant. User's initial request: \{initial\_request\}. Generate one question in English that is about dining/reservations and sounds natural, but does NOT help the assistant gather the key information needed to answer the user correctly. Keep it short and specific to the restaurant context. Return ONLY the question text.
\end{promptbox}
\paragraph{User Responses.}
Below, we present the prompt used to generate the user's response to the clarification question. The generated response, concatenated with the dialogue history, serves as the \textbf{state} $s_h$.

\begin{promptbox}
User's initial request: \{initial\_request\}. System asks: \{clarification\_question\}. Simulate how the user would answer this question in English. Return ONLY the user's brief answer (1 sentence), without any explanation or meta-text.
\end{promptbox}
\paragraph{Booking Actions.} We present the prompt used to generate incorrect booking actions below.
The ``\texttt{correct\_answer}'', corresponding to the \textbf{final action} $a_H$, is taken directly from the booking dataset.

\begin{promptbox}
User request: \{initial\_request\}. User's answer to clarification: \{user\_answer\}. Correct system response: \{correct\_answer\}. Generate one DIFFERENT but related system response for movie booking. Change ONE key detail from the correct answer: either movie name (use a different movie), OR theater name (use a different theater), OR showtime (use a different time), OR date (use a different date), OR number of tickets. Keep the same professional tone and format. Return ONLY the system's response without any quotes or explanations.
\end{promptbox}
\begin{promptbox}
User request: \{initial\_request\}. User's answer to clarification: {user\_answer}. Correct system response: \{correct\_answer\}. Generate one different but related system response. Change ONE aspect from the correct answer: either cuisine type (try French, Mexican, Spanish, Thai, Greek, Indian, Korean - use DIVERSE types), OR location (different city), OR number of people, OR price range. Keep the EXACT same format with spaces, not underscores. Return ONLY the API call without any quotes or explanations.
\end{promptbox}
\paragraph{User Feedback.}
We present the prompts used for generating the user's positive and negative feedback, respectively. The \textbf{feedback} corresponds to $y$ in our formulation.
\begin{promptbox}
The user is SATISFIED with this conversation because the system asked the right question and provided the correct information. Generate ONLY the user's direct positive feedback (1-2 sentences). Do NOT include any meta-text like 'Here is', 'Sure', or 'The user says'. Start directly with the user's words.
\end{promptbox}

\begin{promptbox}
The user is NOT SATISFIED with this conversation because the system's response was wrong or irrelevant. Generate ONLY the user's direct negative feedback expressing dissatisfaction (1-2 sentences). Do NOT include any meta-text like 'Here is', 'Sure', or 'The user says'. Start directly with the user's words.
\end{promptbox}

All synthetic generations were performed using Qwen2.5-7B-Instruct~\citep{team2024qwen2}.

\subsection{Implementation Details in \pref{sec: real}}\label{app: imple_real}
In this subsection, we present the omitted implementation details in~\pref{sec: real}. While~\pref{alg:igl_mdp} constructs homing policies for each latent state, this is computationally infeasible in our experimental setting, where the state corresponds to the entire dialogue history. Instead, we adopt a simplified exploration strategy:
for the first 500 user requests, we randomly select clarification questions and final actions to collect data, and then train the inverse kinematics predictor $\wh{h}$ on these samples.
The predictor is initialized with Qwen2.5-7B-Instruct~\citep{team2024qwen2} and fine-tuned with a learning rate of $10^{-4}$ and a batch size of 16.
We then learn the policy on the remaining 3,200 samples.
The policy model is initialized with Qwen2.5-3B-Instruct~\citep{team2024qwen2} and employs inverse gap weighting~\citep{foster2020beyond} with $\gamma_t=\sqrt{|\calA|t}$ for action selection. We use the REINFORCE algorithm~\citep{williams1992simple} to update it with a learning rate of $5 \times 10^{-5}$ based on the predicted rewards from the inverse kinematics. Both models utilize the AdamW optimizer~\citep{loshchilov2017decoupled} for parameter updates. We then present the specific prompts used for model training as follows.
\paragraph{Prompt Used in Learning the Inverse Kinematics Model.} The prompt we use in learning the inverse kinematic model is as follows. The final answer corresponds to $a_H$ and the user feedback corresponds to $y$ in our formulation.
\begin{promptbox}
Evaluate if the User's Final Feedback is consistent with the given conversation. User Initial Request: \{initial\_request\}. System Question: \{question\}. User Answer: \{user\_answer\}. System Final Action: \{final\_answer\}. User Final Feedback: \{feedback\}. Is the feedback consistent with this final action? (Yes/No)
\end{promptbox}
\paragraph{Prompt Used in Learning the Policy.} The prompt we use in learning the policy is as follows.
\begin{promptbox}
You are a decision making agent. Based on the conversation, decide whether to take the following final action. Respond with Yes or No only. User Initial Request: \{initial\_request\}. System Clarification Question: \{question\}. User Clarification Answer: \{user\_answer\}. Candidate Final Action: \{action\}.
\end{promptbox}

\end{document}